\documentclass[10pt]{article} 
\usepackage[preprint]{tmlr}

\usepackage{hyperref}
\usepackage{url}
\usepackage{graphicx}
\usepackage{booktabs}
\usepackage{multirow}
\usepackage{amsfonts,amssymb,amsthm}
\usepackage{mathtools,bm,xparse,thmtools}
\usepackage{thm-restate}
\usepackage[capitalize]{cleveref}
\usepackage[ruled,vlined,linesnumbered]{algorithm2e}
\usepackage{pifont}
\usepackage{subcaption}
\usepackage{makecell}
\usepackage{tablefootnote}

\usepackage{tcolorbox}
\tcbuselibrary{breakable}
\newenvironment{shadedTheorem}[1][]{%
    \begin{tcolorbox}[
        breakable,
        colback=gray!5,
        colframe=gray!50,
        #1
    ]
}{%
    \end{tcolorbox}
}

\usepackage{tikz}
\usetikzlibrary{arrows.meta}
\usetikzlibrary{positioning}
\usetikzlibrary{calc}

\usepackage{etoolbox}
\robustify\hyperlink

\usepackage{amsmath,amsfonts,bm}

\def\eqref#1{equation~\ref{#1}}

\def\1{\bm{1}}

\def\vp{{\bm{p}}}
\def\vq{{\bm{q}}}
\def\vr{{\bm{r}}}

\def\mP{{\bm{P}}}
\def\mQ{{\bm{Q}}}
\def\mR{{\bm{R}}}

\DeclareMathAlphabet{\mathsfit}{\encodingdefault}{\sfdefault}{m}{sl}
\SetMathAlphabet{\mathsfit}{bold}{\encodingdefault}{\sfdefault}{bx}{n}

\def\gL{{\mathcal{L}}}

\def\gO{{\mathcal{O}}}

\def\gX{{\mathcal{X}}}

\newcommand{\E}{\mathbb{E}}

\newcommand{\R}{\mathbb{R}}

\newcommand{\Var}{\mathrm{Var}}

\newcommand{\KL}[2]{D_{\mathrm{KL}}\left(#1 \middle\| #2 \right)} 
\newcommand{\SRFE}[2]{D_{\mathrm{SRFE}}^\tau\left(#1 \middle\| #2 \right)} 
\newcommand{\CR}[2]{D_{\mathrm{CR}}^\tau\left(#1 \middle\| #2 \right)} 
\newcommand{\CRlambda}[2]{D_{\mathrm{CR}}^{\lambda}\left(#1 \middle\| #2 \right)} 
\newcommand{\Excess}{\mathsf{Excess}}

\newcommand\pcref[1]{(\cref{#1})}

\theoremstyle{plain}
\newtheorem{theorem}{Theorem}[section]

\newtheorem{proposition}[theorem]{Proposition}
\newtheorem{corollary}[theorem]{Corollary}

\theoremstyle{definition}
\newtheorem{definition}[theorem]{Definition}

\crefname{appendix}{Appendix}{Appendices}
\Crefname{appendix}{Appendix}{Appendices}
\crefname{subappendix}{Appendix}{Appendices}
\Crefname{subappendix}{Appendix}{Appendices}

\title{Surprisal-R\'enyi Free Energy}

\author{\name Shion Matsumoto \email matsumoto@usf.edu \\
    \addr Bellini College of Artificial Intelligence, Cybersecurity and Computing \\
    University of South Florida
    \AND
    \name Benjamin Prada \email bprada@usf.edu \\
    \addr Bellini College of Artificial Intelligence, Cybersecurity and Computing \\
    University of South Florida
    \AND
    \name Raul Castillo \email castillo113@usf.edu \\
    \addr Bellini College of Artificial Intelligence, Cybersecurity and Computing \\
    University of South Florida
    \AND
    \name Ankur Arjun Mali \email ankurarjunmali@usf.edu \\
    \addr Bellini College of Artificial Intelligence, Cybersecurity and Computing \\
    University of South Florida
}

\def\month{MM}  
\def\year{2026} 
\def\openreview{\url{https://openreview.net/forum?id=XXXX}} 

\begin{document}

\maketitle

\begin{abstract}
The forward and reverse Kullback--Leibler (KL) divergences arise as limiting objectives in learning and inference, yet they induce markedly different inductive biases that cannot be explained by expectation terms alone. In this work, we introduce the \emph{Surprisal--R\'enyi Free Energy} (SRFE), a normalized R\'enyi--Chernoff log-moment functional of the likelihood ratio. Although SRFE is algebraically related to R\'enyi divergence for fixed values of the order parameter, its endpoint-symmetric normalization yields a distinct free-energy viewpoint for analyzing how forward- and reverse-KL behavior emerges near singular limits. We show that SRFE recovers the forward and reverse KL divergences as endpoint limits and derive local expansions around both regimes in which the variance of the log-likelihood ratio appears as a first-order correction away from the corresponding KL objective. These expansions reveal an explicit mean--variance tradeoff governing departures from KL-dominated behavior. We further establish a Gibbs-type variational characterization in terms of weighted KL projections and prove that SRFE controls large deviations of excess code length through Chernoff-type bounds, giving a Minimum Description Length interpretation of the functional. We study the properties of SRFE in two synthetic settings (two-dimensional Gaussian-mixture variational inference and multi-class classification) and evaluate SRFE on the long-tailed CIFAR-10 and CIFAR-100 datasets through a hyperparameter optimization study. On long-tailed CIFAR-10 and CIFAR-100, retraining the best configuration of each loss over five seeds, SRFE matches cross-entropy accuracy on CIFAR-10-LT and gives mixed results on CIFAR-100-LT, improving few-shot accuracy at one of three imbalance factors. Together, these results position SRFE as a variance- and tail-sensitive, normalized R\'enyi--Chernoff free energy that clarifies the geometric, statistical, and large-deviation structure underlying the transition between forward- and reverse-KL limits.
\end{abstract}

\section{Introduction}

Approximating an intractable probability distribution $P$ by a tractable parametric distribution $Q_\theta$ is a central problem in probabilistic machine learning (ML). Given densities $p$ and $q_\theta$ with respect to a common base measure, the objective is to choose parameters $\theta\in\Theta$ that minimize a discrepancy of the form $D(P\|Q_\theta)$.
This formulation underlies a wide range of learning and inference problems, including supervised classification \citep{krizhevsky2012}, variational inference \citep{kingma2022}, generative modeling \citep{goodfellow2014,nowozin2016}, knowledge distillation \citep{hinton2015distill}, and reinforcement learning \citep{schulman2015,schulman2017,ouyang2022}. Although many divergence families have been proposed, much of modern probabilistic ML is governed by two asymmetric limiting choices: the forward Kullback-Leibler (KL) divergence $\KL{P}{Q_\theta}$ \citep{kullback1951} and the reverse KL divergence $\KL{Q_\theta}{P}$.

The forward KL divergence is most familiar and is used in supervised learning through the cross-entropy (CE) objective. Its penalty becomes large when $q_\theta$ assigns too little probability to events that occur under $P$. As a result, forward KL tends to encourage coverage of the support of $P$, even if this requires assigning probability mass to regions where $P$ has little to no mass in what is known as \emph{mass-covering} behavior \citep{minka2005}. In generative modeling, such behavior can lead to samples that receive high likelihood under the learned model while appearing unrealistic under the data-generating distribution \citep{huszar2015,theis2016a}.

The reverse KL divergence induces a different bias. Since its expectation is taken under $Q_\theta$, it strongly penalizes placing model mass in regions where $P$ assigns little probability. Consequently, reverse KL tends to avoid unsupported regions, but it may do so by concentrating on only a subset of the modes of $P$. This behavior is often described as \emph{mode-seeking} \citep{bishop2006}. Related collapse phenomena arise in generative modeling, including the tendency of adversarially trained models to concentrate on limited regions of the target distribution \citep{salimans2016}\footnote{The original GAN objective is based on the Jensen--Shannon divergence, but the observed collapse behavior is often discussed in relation to mode-seeking effects associated with reverse-KL-like objectives.}. Similar concerns have also appeared in reinforcement learning fine-tuning of large language models, where KL-regularized objectives can introduce strong distributional biases \citep{schulman2017,li2025}. \cref{fig:forward-reverse-kl} illustrates the contrast between mass-covering and mode-seeking behavior when fitting a mixture of Gaussians by a single Gaussian.

\begin{figure}[htbp]
    \centering
    \includegraphics[width=\linewidth]{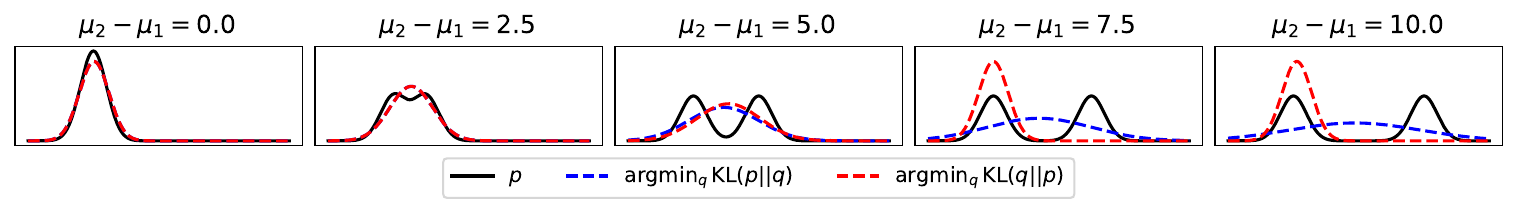}
    \caption[Gaussian $Q$ that minimizes forward KL (blue) and reverse KL (red) with a mixture of Gaussians $P$ with means $\mu_1, \mu_2$ where $\mu_2\ge \mu_1$ and variance $\sigma_1^2=\sigma_2^2=1$. Gaussians are equally weighted.]{Gaussian $Q$ that minimizes forward KL (blue) and reverse KL (red) with a mixture of Gaussians $P$ with means $\mu_1, \mu_2$ where $\mu_2 \ge \mu_1$ and variance $\sigma_1^2=\sigma_2^2=1$. Gaussians are equally weighted.}
    \label{fig:forward-reverse-kl}
\end{figure}

These two KL objectives share the same global minimizer when the model class is well specified, but they impose different inductive biases whenever the approximation family is restricted. The practical choice of objective therefore affects not only convergence, but also which parts of the target distribution are emphasized, which modes are ignored, and how tail events are treated. This suggests that the distinction between forward and reverse KL is not fully explained by their expectation terms alone. A finer analysis should also account for fluctuations of the log-likelihood ratio and the large-deviation behavior of excess code length.

\paragraph{Contributions}

Motivated by this perspective, we introduce and study the \emph{Surprisal-R\'enyi Free Energy} (SRFE), a normalized R\'enyi-Chernoff log-moment functional of the likelihood ratio. For probability distributions $P$ and $Q$ with densities $p$ and $q$, SRFE is defined through the Chernoff coefficient \citep{nielsen2011} $\int p(x)^\tau q(x)^{1-\tau}\,d\mu(x)$ for $\tau\in(0,1)$ with an endpoint-symmetric normalization. Thus, for fixed $\tau$, SRFE is algebraically related to the R\'enyi divergence \citep{Renyi1961}. Its purpose is \emph{not} to define a divergence family unrelated to R\'enyi or Chernoff divergences, but to provide a free-energy normalization that makes the two KL endpoint regimes, their variance corrections, and their tail behavior analytically explicit.

Under the convention above, SRFE recovers the reverse KL divergence $\KL{Q}{P}$ as $\tau\to 0$ and the forward KL divergence $\KL{P}{Q}$ as $\tau\to 1$. For intermediate values of $\tau$, the functional traces a normalized R\'enyi--Chernoff path between these two limiting objectives. This path exposes how departures from KL-dominated behavior depend not only on the mean log-likelihood ratio, but also on its variance and higher cumulants. In this sense, SRFE provides a risk-sensitive free-energy viewpoint on the transition between the mass-covering and mode-seeking behavior of the forward and reverse KL divergences.

In summary, our contributions are as follows:
\begin{enumerate}

\item We formulate SRFE as a normalized R\'enyi--Chernoff log-moment functional of the likelihood ratio and establish its basic properties, including nonnegativity, endpoint KL limits, and its relationship to R\'enyi, Chernoff, and $f$-divergence constructions \pcref{sec:properties}.

\item We derive local expansions of SRFE near both KL endpoints. These expansions show that the variance of the log-likelihood ratio appears as the first correction away from the corresponding KL objective, yielding an explicit mean--variance description of departures from forward- and reverse-KL behavior \pcref{sec:second-order-analysis}.

\item We derive the gradient form of SRFE and show that it can be written as an expectation under an escort distribution. This representation clarifies how the order parameter $\tau$ changes the effective sampling distribution used by the gradient and thereby modifies the optimization dynamics between the two KL regimes \pcref{sec:gradient-dynamics}.

\item We show that SRFE induces the Fisher-Rao Riemannian metric locally under standard regularity assumptions. Thus, while the normalization modifies the global divergence geometry, the local statistical geometry agrees with the intrinsic Fisher structure of the model manifold \pcref{sec:info-geom-unification}.

\item We establish a Minimum Description Length (MDL) interpretation of SRFE by deriving Chernoff-type bounds for excess code length. This shows that SRFE controls large-deviation behavior and provides a risk-sensitive coding interpretation beyond average codelength alone \pcref{sec:mdl}.

\item We evaluate SRFE on synthetic experiments for the continuous and discrete cases and image classification with long-tailed CIFAR-10/100 \pcref{sec:experiments}.

\end{enumerate}

\paragraph{Outline} The remainder of the paper is organized as follows. In \cref{sec:background}, we introduce notation and relevant background. Then, we introduce SRFE in \cref{sec:srfe} along with its basic properties. We then explore various characteristics of the SRFE through analyses of its KL endpoints and gradient dynamics. We conclude the section with information-geometric and MDL-based interpretations of the SRFE-based objective. Following the theoretical results, we present our findings on the empirical effectiveness of the SRFE as a training objective on synthetic benchmarks and long-tail image classification in \cref{sec:experiments}. We provide a discussion of our theoretical and empirical results, limitations, and directions for future work in \cref{sec:discussion}. All proofs are in \cref{app:proofs} and are hyperlinked to their original statement in the main text.

\section{Background} \label{sec:background}

\paragraph{Notation}

Let $P$ and $Q$ denote probability distributions over a measurable space $(\gX,\Sigma)$ with densities $p(x)$ and $q_\theta(x)$ with respect to a $\sigma$-finite reference measure $\mu$.
A measure $P$ is \emph{absolutely continuous} with respect to $Q$ if $P(A) = 0$ whenever $Q(A) = 0$ for all events $A\in\Sigma$. We will denote this as $P\ll Q$. The $\log$ symbol denotes the natural logarithm unless stated otherwise. See \cref{app:notation} for a notation table.

We consider the task of approximating $p$ using a parametric model $q_\theta$ drawn from a predefined family (e.g., Gaussian). Our objective is to minimize a divergence measure $D(P\|Q_\theta)$.

\begin{definition}[Forward and reverse KL divergences]
The forward KL divergence is defined as
\begin{equation}
\KL{P}{Q}
=
\int_\gX p(x)\log\frac{p(x)}{q_\theta(x)}\,d\mu(x),
\, \text{when } P\ll Q,
\end{equation}
and the reverse KL divergence as
\begin{equation}
\KL{Q}{P}
=
\int_\gX q_\theta(x)\log\frac{q_\theta(x)}{p(x)}\,d\mu(x),
\, \text{when } Q\ll P.
\end{equation}
\end{definition}

\paragraph{\texorpdfstring{$f$}{f}-divergences}

A broad class of divergences is given by the $f$-divergence family \citep{csiszar1963}.

\begin{definition}[$f$-divergence]
\label{def:f-divergence}
Let $f:(0,\infty)\to\mathbb{R}$ be convex with $f(1)=0$. 
The $f$-divergence from $P$ to $Q$ is
\begin{equation}
D_f(P\|Q)
=
\int_\gX q(x)\,f\!\left(\frac{p(x)}{q(x)}\right)d\mu(x).
\end{equation}
\end{definition}

By appropriate choice of the generator $f$, this class recovers classical divergences such as the KL divergence, Jensen-Shannon divergence, $\chi^2$ divergence, and Hellinger distance.

\paragraph{Cressie-Read power divergence family}

A notable one-parameter subset of $f$-divergences is the Cressie-Read power divergence family (CR) \citep{cressie1984}.

\begin{definition}[Cressie-Read power divergence family] \label{def:cr-divergence}
Assume $P \ll Q$. For $\lambda\in\mathbb{R}\setminus\{0,-1\}$,
\begin{equation}
\label{eq:cr-divergence}
\CRlambda{P}{Q}
=
\frac{1}{\lambda(\lambda+1)}
\int_\gX
p(x)\!\left[
\left(\frac{p(x)}{q(x)}\right)^\lambda - 1
\right] d\mu(x).
\end{equation}
The cases $\lambda=0$ and $\lambda=-1$ are defined by continuous limits.
\end{definition}

Although CR was originally introduced for goodness-of-fit testing, here, we are particularly interested in its connection to the forward and reverse KL divergences.

\begin{restatable}[\hyperlink{CR-KL-limits}{CR limits to KL divergences}]
{theorem}{thmcrkllimits} \label{thm:CR-KL-limits}
If $P \ll Q$, then $\lim_{\lambda\to 0} \CRlambda{P}{Q} = \KL{P}{Q}.$
If, in addition, $Q \ll P$, then $\lim_{\lambda\to -1} \CRlambda{P}{Q} = \KL{Q}{P}.$
\end{restatable}

\paragraph{R\'enyi divergence} We introduce another parametric family of divergences that generalizes the KL divergence known as the R\'enyi divergence \citep{Renyi1961}.

\begin{definition}[R\'enyi divergence] \label{def:renyi-divergence}
Assume $P \ll Q$. Then, for $0<\alpha<\infty$ and $\alpha\neq1$ the R\'enyi divergence of order $\alpha$ (also known as the $\alpha$-divergence) of $P$ from $Q$ is defined as
\begin{equation}
\label{eq:renyi-divergence}
D_\alpha(P\|Q)
=
\frac{1}{\alpha-1}
\log\int_\gX p(x)^\alpha q(x)^{1-\alpha} d\mu(x).
\end{equation}
The values at $\alpha=0,1,\infty$ are defined by taking the limit. In particular, $\lim_{\alpha\uparrow 1} D_\alpha(P\|Q) = \KL{P}{Q}$.
\end{definition}

Like CR, the R\'enyi divergence is a parametric divergence that interpolates between different divergences; however, it does not interpolate between the forward and reverse KL divergences.
The R\'enyi divergence was derived as a generalization of Shannon entropy but has since been adapted to fields such as variational inference \citep{Li2016}, differential privacy \citep{Mironov_2017}, and data distribution shifts \citep{pmlr-v139-zhai21a}. Please see \cref{app:related-work} for an additional discussion of divergences and review of related work.

\paragraph{Motivation for going beyond CR and R\'enyi}

While CR offers a flexible interpolation between the forward and reverse KL divergences, it operates through power moments of the likelihood ratio $p/q$.
Consequently, its behavior is governed by raw moment growth, making it \emph{highly} sensitive to heavy-tailed likelihood ratios that can dominate the objective through high-order power terms.
This observation motivates the development of alternative divergence constructions that operate on the logarithm of the moment-generating function, thereby inducing cumulant-based geometry and more direct control over tail behavior.

Conversely, although the R\'enyi divergence addresses the issue of raw moment growth via the logarithm, it does \emph{not} interpolate between the forward and reverse KL divergences via a single parameter as does CR.
Moreover, its asymmetric scaling factor leads to poorly-behaved gradient dynamics and its empirical performance in approximate inference algorithms using both biased \cite{Geffner2021b} and unbiased \citep{Geffner2021a} gradient estimators have been mostly unsuccessful.
Together, this motivates construction of a log-moment-generating functional that smoothly interpolates between the forward and reverse KL divergences, retains tail sensitivity, and provides robust gradients.

\section{Surprisal-R\'enyi Free Energy} \label{sec:srfe}

We propose the Surprisal-R\'enyi Free Energy (SRFE), a free energy functional based on the R\'enyi divergence. 

\begin{restatable}[\hyperlink{SRFE-def}{SRFE and associated CR}]{definition}{srfedef}
\label{def:srfe}
Fix $\tau\in(0,1)$ and define
\begin{equation} \label{eq:F-tau}
    F(\tau) \coloneq \int_{\mathcal X} p(x)^{\tau}\,q(x)^{1-\tau}\,d\mu(x),
    \quad\text{where}\quad F(\tau)>0,
\end{equation}
which is the Chernoff $\tau$-coefficient \citep{nielsen2011}.
Assume $F(\tau)>0$ so that the quantities below are finite.
Then, the \textbf{Surprisal-R\'enyi Free Energy (SRFE)} is
\begin{equation}
\SRFE{P}{Q}
\coloneq
\frac{-\log F(\tau)}{\tau(1-\tau)} .
\label{eq:srfe}
\end{equation}
The \textbf{associated CR power divergence family} is
\begin{equation}
\CR{P}{Q}
\coloneq
\frac{1 - F(\tau)}{\tau(1-\tau)}.
\label{eq:cr-associated}
\end{equation}
In particular, $\CR{P}{Q}$ coincides with \cref{def:cr-divergence} where $\lambda=\tau-1$. We will refer to the form introduced in \cref{def:cr-divergence} as the \emph{standard} CR and \cref{def:srfe} as the \emph{associated} CR.
\end{restatable}

\paragraph{Almost disjoint supports}

Because of the $\log F(\tau)$ term in SRFE, it is \emph{not} well-defined under \emph{disjoint} supports. When $p$ and $q_\theta$ share no overlap, $F(\tau)=0$ for $\tau\in(0,1)$ (hence $\log F(\tau)=-\infty$) and $F(\tau)=\infty$ for $\tau>1$ (hence $\log F(\tau)=+\infty$), preventing gradient-based optimization.
Importantly, $F(\tau)>0$ does \emph{not} require mutual absolute continuity ($p\ll q_\theta$ and $q_\theta\ll p$); it suffices that $p$ and $q_\theta$ have
\emph{nonzero overlap} (i.e., the set $\{x: p(x)>0,\; q_\theta(x)>0\}$ has positive measure). In practice, model parameterizations and smoothing (e.g., softmax outputs, label smoothing, or additive density floors) typically ensure small but nonzero overlap allowing both SRFE and CR to be well-defined.

\subsection{Basic Properties} \label{sec:properties}

We summarize several key properties of SRFE.

\begin{restatable}[\hyperlink{SRFE-KL-limits}{KL limits}]{theorem}{srfelimits}
\label{thm:srfe-kl-limits}
For $\tau\in(0,1)$ with $P\ll Q$ and $Q\ll P$,
$\lim_{\tau\to 0} \SRFE{P}{Q} = \KL{Q}{P}$ and
$\lim_{\tau\to 1} \SRFE{P}{Q} = \KL{P}{Q}$.
\end{restatable}

\begin{restatable}[\hyperlink{SRFE-nonnegativity}{Nonnegativity}]{lemma}{srfenonneg}
\label{lem:srfe-nonneg}
For $\tau\in(0,1)$, we have $0<F(\tau)\le1$,
with equality $F(\tau)=1$ iff $p=q$ $\mu$-a.e.
Consequently, $\SRFE{P}{Q} \ge 0$
with equality iff $P=Q$.
\end{restatable}

\begin{restatable}[\hyperlink{SRFE-monotone}{Monotone equivalence with CR}]{lemma}{monotone}
\label{lem:srfe-cr-monotone}
For $\tau\in(0,1)$, both CR and SRFE are strictly monotone transforms of the Chernoff $\tau$-coefficient $F(\tau)$.
Hence they induce identical orderings over models and share the same unique minimizer $p=q$.
\end{restatable}

Therefore, while CR and SRFE agree at the level of model ordering, they differ structurally -- whereas CR depends on raw power moments of the likelihood ratio, SRFE depends on the logarithm of the MGF.

\begin{restatable}[\hyperlink{SRFE-not-f-div}{SRFE is not an $f$-divergence}]{theorem}{srfenotfdiv}
\label{thm:srfe-not-fdiv}
For $\tau\in(0,1)$, SRFE cannot be expressed in the form
$\sum_i q_i f(p_i/q_i)$,
and therefore is not an $f$-divergence.
\end{restatable}

This distinction is significant: whereas CR belongs to the $f$-divergence family and is governed by power-moment behavior, SRFE operates on the log-MGF, thereby inducing a cumulant-based geometry. These properties establish that SRFE preserves the essential divergence characteristics of CR while departing structurally from the $f$-divergence class, thereby motivating separate geometric and optimization analysis.

\subsection{Second-Order Surprisal Analysis} \label{sec:second-order-analysis}

We now compare the second-order behavior of CR and SRFE to provide further insight into their properties. From here on out, we define the log-likelihood ratio (surprisal gap), $\Delta(x) \coloneq\log \frac{p(x)}{q(x)}$.

\begin{restatable}[\hyperlink{CR-surprisal-variance}{Surprisal-variance expansion for standard CR}]{theorem}{crsurprisal}
\label{thm:CR-surprisal-variance}
Let $P\ll Q$ and assume
$\E_{P}[|\Delta|^3]<\infty$. For $\tau\neq 0,-1$ and the standard CR as in \cref{eq:cr-divergence},
we have, as $\tau\to 0$,
\begin{equation}
\label{eq:CR-surprisal-variance}
\CR{P}{Q}
=
\KL{P}{Q}
+\frac{\tau}{2}\Var_{P}[\Delta]
+\tau\Bigl(\tfrac12\KL{P}{Q}^2-\KL{P}{Q}\Bigr)
+O(\tau^2).
\end{equation}
In particular, the linear correction contains the term $\frac{\tau}{2}\Var_{P}[\Delta]$.
\end{restatable}

We similarly perform a second-order (local) expansion of SRFE around the forward and reverse KL limits.

\begin{restatable}[\hyperlink{SRFE-surprisal-forward}{SRFE local expansion around forward KL}]{theorem}{srfeforwardexpansion}
\label{thm:srfe-local-forward-KL}
Let $P\ll Q$ and assume $\E_{P}[\Delta^2]<\infty$ and that differentiation under the integral is permitted for $\int p(x)^\tau q(x)^{1-\tau}$
in a neighborhood of $\tau=1$. Then, as $\tau\uparrow 1$,
\begin{equation}
\label{eq:srfe-taylor-forward}
\SRFE{P}{Q}
=
\KL{P}{Q}
+
(1-\tau)\Bigl(\KL{P}{Q} - \tfrac12 \Var_{P}[\Delta]\Bigr)
+
O\bigl((1-\tau)^2\bigr).
\end{equation}
\end{restatable}

\begin{restatable}[\hyperlink{SRFE-surprisal-backward}{SRFE local expansion around reverse KL}]{theorem}{srfebackwardexpansion}
\label{thm:srfe-local-backward-KL}
Let $Q\ll P$ and assume $\E_{P}[\Delta^2]<\infty$ and that differentiation under the integral is permitted for $\int p(x)^\tau q(x)^{1-\tau}$
in a neighborhood of $\tau=0$. Then, as $\tau\downarrow 0$,
\begin{equation}
\label{eq:srfe-taylor-backward}
\SRFE{P}{Q}
=
\KL{Q}{P}
+
\tau\Bigl(\KL{Q}{P} - \tfrac12 \Var_{Q}[\Delta]\Bigr)
+
O\bigl(\tau^2\bigr).
\end{equation}
\end{restatable}

\paragraph{Interpretation}

\Cref{eq:srfe-taylor-forward,eq:srfe-taylor-backward} show that SRFE recovers $\KL{P}{Q}$ as $\tau\uparrow 1$ and $\KL{Q}{P}$ as $\tau\downarrow 0$, with a first-order correction proportional to the variance of the log-likelihood ratio $\Delta(x)$. 
This variance term captures second-order fluctuations of the excess codelength (surprisal) beyond the mean mismatch encoded by KL, and explains how $\tau$ tunes sensitivity to dispersion/tail behavior near the KL endpoints.


While the associated CR divergence also admits a second-order expansion involving the variance of the log-likelihood ratio, this variance arises from a power-moment expansion of the likelihood ratio. In contrast, SRFE is defined through the logarithm of the MGF, so its curvature is governed by cumulants of the log-likelihood ratio rather than raw moments. Consequently, although the local Taylor expansions of CR and SRFE contain similar variance terms, SRFE induces a different global geometry with direct connections to large-deviation behavior and tail sensitivity.

\subsection{Gradient Dynamics} \label{sec:gradient-dynamics}

Next, we compare the learning dynamics for SRFE and CR, particularly in the almost disjoint setting, by deriving their gradients. For \cref{lem:srfe-gradient,lem:cr-gradient}, let $q_\theta$ be differentiable in $\theta$ and assume differentiation can be interchanged with integration.

\begin{restatable}[\hyperlink{SRFE-gradient}{SRFE gradient}]{lemma}{srfegrad}
\label{lem:srfe-gradient}
For $\tau\in(0,1)$,
\begin{equation} \label{eq:srfe-gradient}
\nabla_\theta \SRFE{P}{Q_\theta}
=
-\frac{1}{\tau}\,
\E_{x\sim r_\tau}\!\left[\nabla_\theta \log q_\theta(x)\right],
\end{equation}
where
\begin{equation}
    r_\tau(x)\coloneq\frac{p(x)^\tau q_\theta(x)^{1-\tau}}{F(\tau)}.
\end{equation}
\end{restatable}

\begin{restatable}[\hyperlink{CR-gradient}{Associated CR gradient}]{lemma}{crgrad}
\label{lem:cr-gradient}
For $\tau\in(0,1)$ and $u(x)=p(x)/q_\theta(x)$,
\begin{equation} \label{eq:cr-gradient}
\nabla_\theta \CR{P}{Q_\theta}
=
-\frac{1}{\tau}\,
\E_{x\sim Q_\theta}\!\left[u(x)^\tau \,\nabla_\theta \log q_\theta(x)\right].
\end{equation}
Equivalently (baseline-subtracted form),
\begin{equation} \label{eq:cr-gradient-subtracted}
\nabla_\theta \CR{P}{Q_\theta}
=
-\frac{1}{\tau}
\E_{x\sim Q_\theta}\left[(u(x)^\tau-1)\nabla_\theta \log q_\theta(x)\right].
\end{equation}
\end{restatable}

\paragraph{Optimization advantages of SRFE}

Although SRFE does not eliminate the absolute-continuity requirement shared by divergence-based objectives,
it offers distinct optimization advantages over classical $f$-divergences.
Unlike CR or KL-type objectives, SRFE gradients do not contain explicit likelihood-ratio terms such as $u(x)^\tau$
inside the expectation; instead, the density ratio $p(x)/q_\theta(x)$ influences optimization only through escort-distribution weights $r_\tau(x)$.
For $\tau\in(0,1)$, these weights suppress low-density regions of $Q_\theta$, acting as an implicit trust region that prevents gradient amplification as $q_\theta(x)\to 0$.
As a result, SRFE yields better-conditioned, lower-variance gradients in the almost disjoint regime and
admits a continuous interpolation between conservative and aggressive update behavior via $\tau$.
These properties improve optimization stability without requiring ad-hoc clipping or regularization.

Next, we provide bounds on second moments. A finite second moment guarantees finite variance as
$\Var(g)\le \mathbb{E}\|g\|^2$,
while divergence of the second moment implies divergence of the variance.
Second-moment bounds are therefore sufficient to characterize stability and avoid explicitly computing $\mathbb{E}[g]$, which is often intractable.

\begin{restatable}[\hyperlink{CR-gradient-estimator}{Associated CR gradient estimator and second-moment bound}]{lemma}{crsecondmoment}
\label{lem:cr-second-moment}
Let $q_\theta$ be a smooth density model and assume
$\|\nabla_\theta \log q_\theta(x)\|^2 \le C$ for all $x$.
Fix $\tau\in(0,1)$ and define $u(x)=p(x)/q_\theta(x)$.
The unbiased one-sample stochastic gradient estimator for the associated CR under $X\sim Q_\theta$ is
\begin{equation}
\label{eq:cr-stochastic-gradient}
g_\mathrm{CR}(X)
:=
-\frac{1}{\tau}\,u(X)^\tau\,\nabla_\theta \log q_\theta(X).
\end{equation}
Moreover,
\begin{equation}
\label{eq:cr-second-moment}
\mathbb{E}_{X\sim Q_\theta}\!\big[\|g_\mathrm{CR}(X)\|^2\big]
\le
\frac{C}{\tau^2}
\int_{\gX} q_\theta(x)\,u(x)^{2\tau}\,d\mu(x),
\end{equation}
which diverges whenever $q_\theta(x)\to 0$ on sets where $p(x)>0$ and the integral is infinite.
\end{restatable}

\begin{restatable}[\hyperlink{SRFE-gradient-estimator}{SRFE gradient estimators and second-moment bounds}]{lemma}{srfesecondmoment}
\label{lem:srfe-second-moment}
Let $q_\theta$ be differentiable in $\theta$ and assume differentiation can be interchanged with integration.
Fix $\tau\in(0,1)$ and assume $\|\nabla_\theta \log q_\theta(x)\|^2 \le C$ for all $x$.
Define
\(
F(\tau)=\int_{\gX} p(x)^\tau q_\theta(x)^{1-\tau}\,d\mu(x),
\)
\(
r_\tau(x)=\frac{p(x)^\tau q_\theta(x)^{1-\tau}}{F(\tau)},
\)
and $u(x)=p(x)/q_\theta(x)$.

\noindent
\textbf{(i) Escort-sampling estimator.}
If $X\sim r_\tau$, define
\begin{equation}
\label{eq:srfe-escort-estimator}
g_{\mathrm{SRFE}}^{(r)}(X):=
-\frac{1}{\tau}\,\nabla_\theta \log q_\theta(X).
\end{equation}
Then $g_{\mathrm{SRFE}}^{(r)}$ is unbiased for $\nabla_\theta \SRFE{P}{Q_\theta}$ and satisfies
\begin{equation}
\label{eq:srfe-escort-second-moment}
\mathbb{E}_{X\sim r_\tau}\!\big[\|g_{\mathrm{SRFE}}^{(r)}(X)\|^2\big]
\le
\frac{C}{\tau^2}.
\end{equation}

\noindent
\textbf{(ii) $Q_\theta$-sampling (importance-weighted) estimator.}
If $X\sim Q_\theta$, define
\begin{equation}
\label{eq:srfe-q-estimator}
g_{\mathrm{SRFE}}^{(q)}(X):=
-\frac{1}{\tau}\,\frac{u(X)^\tau}{F(\tau)}\,\nabla_\theta \log q_\theta(X).
\end{equation}
Then $g_{\mathrm{SRFE}}^{(q)}$ is unbiased for $\nabla_\theta \SRFE{P}{Q_\theta}$ and satisfies
\begin{equation}
\label{eq:srfe-q-second-moment}
\mathbb{E}_{X\sim Q_\theta}\!\big[\|g_{\mathrm{SRFE}}^{(q)}(X)\|^2\big]
\le
\frac{C}{\tau^2 F(\tau)^2}
\int_{\gX} q_\theta(x)\,u(x)^{2\tau}\,d\mu(x).
\end{equation}
\end{restatable}

When $p(x)>0$ but $q_\theta(x)\to 0$, the density ratio $u(x)$ grows and causes the second moment (and hence the variance) of CR gradient estimators to diverge. In contrast, SRFE admits an escort-based estimator with uniformly bounded second moment, and even when expressed via $Q_\theta$-sampling, its second moment is normalized by $F(\tau)^{-2}$. Consequently, in almost disjoint regimes, SRFE yields strictly better-conditioned gradient estimates, particularly for $\tau<\tfrac12$.

\subsection{Information-Geometric Unification} \label{sec:info-geom-unification}

We now show that SRFE admits both a global and local information-geometric characterization. Globally, SRFE can be expressed as a variational projection onto the exponential (Chernoff) path connecting $P$ and $Q$ \citep{nielsen2022}. For fixed $\tau\in(0,1)$, $\SRFE{P}{Q}$ equals a weighted KL projection onto an intermediate distribution $r_\tau \propto p^\tau q^{1-\tau}$. This distribution lies on the exponential geodesic between $P$ and $Q$, and the resulting decomposition yields a Pythagorean identity in KL geometry.

Locally, SRFE induces the same Riemannian metric as KL. The second-order expansion around $\theta'=\theta$ recovers the Fisher information matrix, independently of $\tau$. Thus, although SRFE modifies the global geometry of the divergence landscape, it preserves the intrinsic statistical manifold structure.

\begin{restatable}[\hyperlink{SRFE-variational}{Variational characterization of SRFE (information-geometric form)}]{theorem}{srfevariational}
\label{thm:srfe-variational}

\begin{equation}
\label{eq:srfe-variational}
\SRFE{P}{Q}
=
\min_{r\in\mathcal P(\mathcal X)}
\left\{
\frac{1}{\tau}\KL{r}{Q}
+
\frac{1}{1-\tau}\KL{r}{P}
\right\},
\end{equation}
and the unique minimizer is the escort (Chernoff) distribution
\begin{equation}
\label{eq:escort}
r_\tau(x)
=
\frac{p(x)^\tau q(x)^{1-\tau}}{\int_{\mathcal X} p(u)^\tau q(u)^{1-\tau}\,d\mu(u)}.
\end{equation}
Equivalently, the objective in \cref{eq:srfe-variational} admits the decomposition
\begin{equation}
\label{eq:pythagorean}
\frac{1}{\tau(1-\tau)}\KL{r}{r_\tau}
+
\SRFE{P}{Q},
\end{equation}
which exhibits an information-geometric (Pythagorean) relation with
projection onto the $\tau$-escort path.
\end{restatable}

\begin{restatable}[\hyperlink{SRFE-riemannian}{Riemannian metric induced by SRFE equals Fisher-Rao}]{theorem}{srferiemannian}
\label{thm:srfe-metric}
Let $\{p_\theta:\theta\in\Theta\subset\mathbb R^d\}$ be a smooth statistical
model on $(\mathcal X,\Sigma,\mu)$ with strictly positive densities.
Fix $\tau\in(0,1)$ and define the SRFE divergence
\begin{equation}
\label{eq:srfe-param}
\SRFE{p_\theta}{p_{\theta'}}
:=
-\frac{1}{\tau(1-\tau)}
\log\!\int_{\mathcal X} p_\theta(x)^\tau p_{\theta'}(x)^{1-\tau}d\mu(x).
\end{equation}
Assume the usual regularity conditions allowing differentiation under the
integral sign and the identities
\[
\E_\theta[\partial_i \log p_\theta]=0,
\qquad
\E_\theta[\partial_{ij}\log p_\theta]= - I_{ij}(\theta),
\]
where $I(\theta)$ is the Fisher information matrix
\[
I_{ij}(\theta):=\E_\theta\!\big[\partial_i\log p_\theta(X)\,\partial_j\log p_\theta(X)\big].
\]
Then the second-order expansion around $\theta'=\theta$ is
\begin{equation}
\label{eq:srfe-local-fisher}
D^{\mathrm{SRFE}}_\tau(p_\theta\|p_{\theta+\delta})
=
\frac12\,\delta^\top I(\theta)\,\delta \;+\; O(\|\delta\|^3).
\end{equation}
Consequently, the Riemannian metric induced by SRFE,
\[
g^{(\mathrm{SRFE})}_{ij}(\theta)
:=
\left.\frac{\partial^2}{\partial \theta'^i\,\partial \theta'^j}
D^{\mathrm{SRFE}}_\tau(p_\theta\|p_{\theta'})\right|_{\theta'=\theta},
\]
coincides with the Fisher-Rao metric:
\begin{equation}
\label{eq:srfe-metric-fisher}
g^{(\mathrm{SRFE})}_{ij}(\theta)=I_{ij}(\theta).
\end{equation}
\end{restatable}

\begin{corollary}[\hyperlink{SRFE-fisher-rao}{$\tau$-independence and Fisher-Rao unification}]
\label{cor:srfe-fisher-unification}
Under the assumptions of \cref{thm:srfe-metric}, the Riemannian metric
induced by $D_{\mathrm{SRFE}}^\tau$ is independent of $\tau\in(0,1)$ and equals
the Fisher-Rao metric:
\[
g^{(\mathrm{SRFE})}_{ij}(\theta)=I_{ij}(\theta)\qquad\text{for all }\tau\in(0,1).
\]
Consequently, for any two smooth curves $\theta(t)$ and $\theta(t)+\delta(t)$
with $\delta(t)\to 0$, the SRFE locally induces the same infinitesimal squared
distance as KL (and hence as any smooth $f$-divergence):
\[
\SRFE{p_\theta}{p_{\theta+\delta}}
=
\frac12\,\delta^\top I(\theta)\,\delta + o(\|\delta\|^2),
\]
uniformly over $\tau$ in compact subsets of $(0,1)$.
\end{corollary}

\subsection{Minimum Description Length (MDL)}
\label{sec:mdl}

We now interpret SRFE through the lens of coding theory and large deviations.
Recall that for densities $p,q$ with $P\ll Q$, the log-likelihood ratio $\Delta(x)$
represents the \emph{excess codelength} (in nats) incurred by encoding $x$ with $Q$ rather than the true distribution $P$.

\begin{restatable}[\hyperlink{SRFE-tailbound}{SRFE controls large deviations of excess codelength}]{theorem}{srfetailbound}
\label{thm:srfe-tailbound}
Fix $\tau\in(0,1)$,
then for any $a\in\mathbb{R}$,
\begin{equation}
\label{eq:srfe-tailbound}
\Pr_{X\sim Q}\!\big[\Delta(X)\ge a\big]
\le
\exp\!\Big(-\tau a - \tau(1-\tau)\SRFE{P}{Q}\Big).
\end{equation}
\end{restatable}

The bound in \cref{eq:srfe-tailbound} is a Chernoff-type large-deviation inequality. Since
$
\tau(1-\tau)\SRFE{P}{Q}
=
-\log \E_Q[e^{\tau\Delta}],
$
SRFE is precisely the normalized log-MGF of the excess codelength.
Thus, SRFE controls the exponential rate at which unlikely compression advantages decay.

\paragraph{Gibbs (Chernoff) variational principle}
The SRFE admits the equivalent variational form
\begin{restatable}[\hyperlink{SRFE-gibbs}{Gibbs variational characterization of SRFE}]{theorem}{srfegibbsvar}
\label{thm:srfe-gibbs-var}
For fixed $\tau\in(0,1)$ and $r\in\mathcal P(\mathcal X)$
\begin{equation}
\label{eq:srfe-gibbs-var}
\SRFE{P}{Q}
=
\min_r
\left\{
\frac{1}{\tau}\KL{r}{Q}
+
\frac{1}{1-\tau}\KL{r}{P}
\right\},
\end{equation}
with a unique minimizer $r_\tau(x)$.
\end{restatable}

This shows that SRFE corresponds to a weighted KL projection onto the exponential
(Chernoff) path connecting $P$ and $Q$.
It therefore measures coding mismatch not only at the mean level (KL),
but along the entire exponential family interpolation.

\paragraph{KL upper bounds}

For all $\tau\in(0,1)$,
\begin{restatable}[\hyperlink{SRFE-KL-bounds}{KL upper bounds for SRFE}]{corollary}{srfeklupper}
\label{cor:srfe-kl-upper}
\begin{equation}
\SRFE{P}{Q}
\le
\min\!\left\{
\frac{1}{1-\tau}\KL{P}{Q},\,
\frac{1}{\tau}\KL{Q}{P}
\right\}.
\end{equation}
\end{restatable}

Thus, SRFE is uniformly controlled by both oriented KL divergences while retaining tail sensitivity.

\paragraph{MDL interpretation}

Define Shannon codelengths $\ell_P(x) := -\log p(x)$, and $\ell_Q(x) := -\log q(x)$.
Then the excess codelength incurred by using $Q$ instead of $P$ is $\Delta(x) = \ell_Q(x) - \ell_P(x)$.

\begin{restatable}[\hyperlink{SRFE-MDL}{MDL interpretation: exponential tail control}]{corollary}{mdlsrfetail}
\label{cor:mdl-srfe-tail}
For any $a\in\mathbb R$,
\begin{equation}
\label{eq:mdl-excess-tail}
\Pr_{X\sim Q}\!\big[\Delta(X)\ge a\big]
\le
\exp\!\Big(-\tau a - \tau(1-\tau)\SRFE{P}{Q}\Big).
\end{equation}
\end{restatable}

Equivalently, the probability that model $Q$ yields a description at least $a$ nats longer than $P$
decays exponentially at a rate governed by $\SRFE{P}{Q}$.

\paragraph{Implications for deep networks}

In deep neural networks, training via maximum likelihood minimizes $\KL{P}{Q_\theta}$, which controls average excess codelength. However, \cref{eq:mdl-excess-tail} shows that SRFE additionally controls \emph{rare but extreme miscalibration events} in which the model assigns exponentially too little probability to true outcomes.
Because SRFE is the log-MGF of the log-likelihood ratio, it penalizes heavy tails in $\Delta(x)$ and thus discourages overconfident errors. This provides a principled connection between SRFE, calibration, and robustness, particularly in overparameterized deep models where rare catastrophic likelihood errors dominate downstream risk.
\section{Experiments} \label{sec:experiments}

We evaluated SRFE as an objective in continuous and discrete synthetic benchmarks and the long-tailed CIFAR-10/100 (CIFAR-LT) dataset \citep{Cui2019}. For the full set of details and results, see \cref{app:experiments-synthetic,app:experiments-cifar}. We provide additional sets of analyses and experiments exploring gradient dynamics \pcref{app:grad-dynamics} and label smoothing \pcref{app:label-smoothing}. We provide implementation details and PyTorch-style pseudocode for SRFE in \cref{app:implementation}.

\paragraph{2D Gaussian Mixture Model}

We illustrate SRFE's ability to interpolate between mode-seeking and mass-covering behavior in the continuous setting with a variational inference benchmark in 2D. The target distribution is a three-component Gaussian mixture with means $\mu_1=[-3, -1]$, $\mu_2=[-4, 2]$, and $\mu_3=[4, y_3]$ where $y_3\in\{4, 0, -3\}$, shared isotropic diagonal covariances $\Sigma_1 = \Sigma_2 = \Sigma_3 = 0.5 I_2$, and mixture weights $w=[0.3, 0.3, 0.4]$.
The variational posterior was set to a Gaussian with diagonal covariance (two trainable parameters $\{\mu,\sigma\}$). We evaluated the forward and reverse KL divergences and SRFE with $\tau=\{0.3, 0.5, 0.7\}$ on 10 seeds.

\begin{figure}[htbp]
    \centering
    \includegraphics[width=\linewidth]{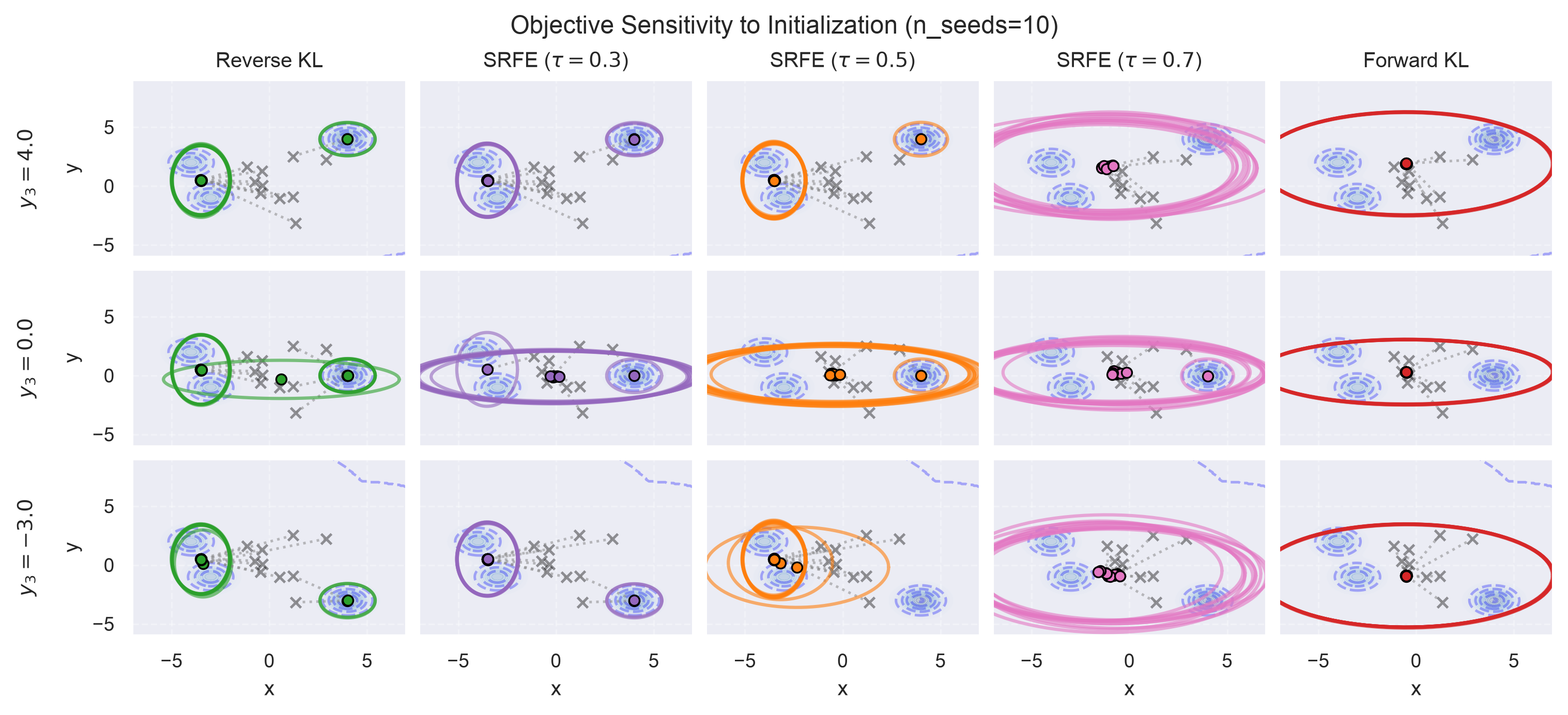}
    \caption{Single Gaussian fitted to three different mixtures of Gaussians under various objectives. Crosses and solid circles represent the initial and converged means, respectively, for 10 initializations.}
    \label{fig:gmm}
\end{figure}

\cref{fig:gmm} illustrates the transition from mode-seeking to mass-covering behavior as $\tau$ goes from 0 (reverse KL) to 1 (forward KL). Reverse KL frequently collapses to one of the two subsets of modes whereas forward KL reliably converges to the same distribution covering all three modes. $\tau=0.3$ and $0.7$ display behavior similar to that of reverse and forward KL, respectively, albeit with higher sensitivity to the initialization. The result for $\tau=0.5$ is varies more widely depending on the target distribution.

\paragraph{Multi-Class Classification}

A small MLP (layers=$[n_\text{dim}, 16, 16, n_\text{classes}]$) was trained on a synthetic 5-class classification dataset using SRFE with $\tau=[0.1, 0.3, 0.5, 0.7, 0.9]$. The dataset was generated using \texttt{scikit-learn}'s \citep{scikit-learn} \texttt{make\_classification} with \texttt{n\_samples=2500}, \texttt{n\_features=12}, \texttt{n\_informative=8}, and \texttt{n\_redundant=2}. Experiments were repeated with 15 seeds.

In the classification setting, SRFE reduces to CE scaled by $1/\tau$ without label smoothing (\cref{app:grad-dynamics}). We therefore trained the MLP with SRFE on label-smoothed targets. The smoothing factor $\alpha$ was set separately for each $\tau$ with $\alpha_\text{eff} = 0.1$ according to \cref{prop:srfe-label-smoothing-exact}. This gives $\alpha \approx 10^{-15}$ at $\tau = 0.1$ up to $\alpha \approx 0.087$ at $\tau = 0.9$.

\begin{figure}[htbp]
    \centering
    \includegraphics[width=\linewidth]{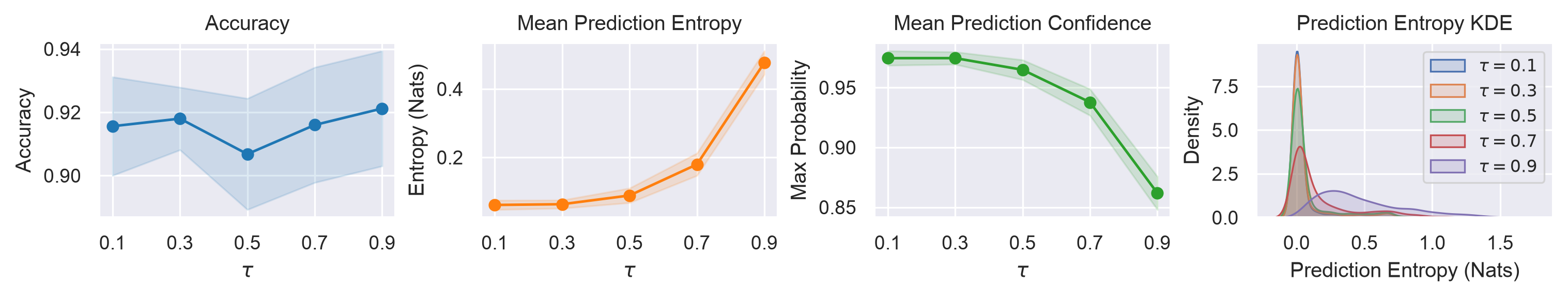}
    \caption{Synthetic 5-class classification with SRFE. (left to right) Mean $\pm$ standard deviation for accuracy, mean prediction entropy, and mean prediction confidence over 15 seeds. KDE of per-sample prediction entropy.}
    \label{fig:discrete-metrics}
\end{figure}

\cref{fig:discrete-metrics} shows that $\tau$ has minimal effect on test accuracy as the means are all roughly 91\% to 92\%. In contrast, prediction entropy and confidence change monotonically with $\tau$. Mean entropy is approximately $0.06$ nats for $\tau \leq 0.3$, rises slightly at $\tau = 0.5$, and increases rapidly to $0.48$ nats at $\tau = 0.9$. The trend is mirrored by the mean confidence, as it falls from $0.97$ to $0.86$ with increasing $\tau$. The prediction entropy kernel density estimation (KDE) plot shows that the majority of predictions have entropy close to zero for $\tau\leq 0.7$. At $\tau = 0.9$, the mode shifts to roughly $0.3$ nats with a longer tail. These results confirm the mode-seeking and mass-covering limits of SRFE (\cref{thm:srfe-local-forward-KL,thm:srfe-local-backward-KL}).

\paragraph{CIFAR-10/100-LT} \label{sec:experiments-cifar-lt}

We compared SRFE against a cross-entropy (CE) baseline on the long-tailed image classification benchmarks CIFAR-10/100-LT \citep{Cui2019}. We evaluated across three levels of class imbalance (IF=\{10, 50, 100\}) using a ResNet-32 architecture \citep{He2016} as is standard in the long-tailed image classification literature \citep{zhou2020BBN,tang2020longtailed,zhang2021tricks}.
We performed a hyperparameter optimization study for SRFE and repeated the study for CE as well to ensure a fair comparison across losses. There were 12 studies in total (2 datasets $\times$ 3 IFs $\times$ 2 objectives). See \cref{app:cifar-setup} for details on the setup.

For each study, we retrained a model using the hyperparameter configuration with the highest validation accuracy on the full, long-tailed training set and evaluated it on the balanced test set (\cref{tab:cifar-lt_eval-results}). We repeated this for 5 seeds. On both datasets, our tuned CE baseline exceeds the results reported by \citet{zhou2020BBN} at every IF.



\begin{table}[htbp]
\centering
\caption{Balanced test set results on CIFAR-10/100-LT. For each dataset, IF, and loss, the configuration with the highest validation accuracy in the hyperparameter study was retrained on the full long-tailed training set with five new seeds. We report the mean $\pm$ standard deviation for accuracy and ECE (15 bins) in \%.
}
\label{tab:cifar-lt_eval-results}
\small
\setlength{\tabcolsep}{4pt}
\begin{tabular}{llccccc}
\toprule
 & & \multicolumn{4}{c}{Accuracy (\%) $\uparrow$} & \\
\cmidrule(lr){3-6}
IF & Loss & Overall & Many & Med & Few & ECE (\%) $\downarrow$ \\
\midrule
\multicolumn{7}{l}{\textit{CIFAR-10-LT}} \\
10 & CE & 87.85 $\pm$ 0.48 & 88.97 $\pm$ 0.33 & 85.25 $\pm$ 1.02 & -- & 5.83 $\pm$ 0.38 \\
 & SRFE & 87.96 $\pm$ 0.28 & 89.13 $\pm$ 0.37 & 85.24 $\pm$ 1.07 & -- & 4.25 $\pm$ 0.43 \\
\addlinespace
50 & CE & 78.77 $\pm$ 0.56 & 90.72 $\pm$ 0.40 & 70.81 $\pm$ 0.92 & -- & 12.94 $\pm$ 0.59 \\
 & SRFE & 79.15 $\pm$ 0.43 & 91.04 $\pm$ 0.58 & 71.23 $\pm$ 1.01 & -- & 10.51 $\pm$ 0.58 \\
\addlinespace
100 & CE & 72.95 $\pm$ 0.65 & 89.72 $\pm$ 0.49 & 69.48 $\pm$ 0.63 & 46.38 $\pm$ 1.89 & 10.89 $\pm$ 0.50 \\
 & SRFE & 73.54 $\pm$ 1.37 & 89.56 $\pm$ 0.26 & 68.86 $\pm$ 1.19 & 50.86 $\pm$ 4.75 & 10.17 $\pm$ 0.69 \\
\midrule
\multicolumn{7}{l}{\textit{CIFAR-100-LT}} \\
10 & CE & 58.82 $\pm$ 0.22 & 65.65 $\pm$ 0.33 & 42.90 $\pm$ 0.80 & -- & 2.28 $\pm$ 0.23 \\
 & SRFE & 57.63 $\pm$ 0.29 & 63.85 $\pm$ 0.42 & 43.10 $\pm$ 0.64 & -- & 16.27 $\pm$ 0.26 \\
\addlinespace
50 & CE & 44.55 $\pm$ 1.07 & 66.62 $\pm$ 0.97 & 37.28 $\pm$ 1.25 & 10.87 $\pm$ 1.87 & 3.12 $\pm$ 0.91 \\
 & SRFE & 46.51 $\pm$ 0.94 & 67.18 $\pm$ 0.71 & 38.73 $\pm$ 1.62 & 17.17 $\pm$ 1.15 & 29.81 $\pm$ 0.80 \\
\addlinespace
100 & CE & 41.61 $\pm$ 0.62 & 69.93 $\pm$ 0.71 & 40.58 $\pm$ 0.74 & 9.78 $\pm$ 0.70 & 15.41 $\pm$ 0.55 \\
 & SRFE & 40.42 $\pm$ 0.74 & 68.43 $\pm$ 0.78 & 38.83 $\pm$ 0.70 & 9.61 $\pm$ 1.41 & 15.93 $\pm$ 0.96 \\
\bottomrule
\end{tabular}
\end{table}

On both CIFAR-10-LT and CIFAR-100-LT, SRFE and CE achieve similar overall, many-shot, and medium-shot accuracy at every IF as the mean accuracies fall within 1 point of each other (apart from CIFAR-100-LT IF=10 and 100 where CE outperforms SRFE by 1.19 points on both). Performances on the many- and medium-shot categories follow a similar trend. On the few-shot categories, however, SRFE outperforms CE by a larger margin for both datasets (4.48 points for CIFAR-10-LT at IF=100 and 6.30 points for CIFAR-100-LT at IF=50, though it is worth noting that the former varied widely across seeds).
In regards to ECE, SRFE matches or outperforms CE for CIFAR-10-LT at at all IFs and CIFAR-100-LT at IF=100. On CIFAR-100-LT at IF=10 and 50, SRFE had a noticeably larger mean ECE, which cannot be attributable to wide fluctuations between seeds as it reported similar levels of variance to other studies.
We additionally report the test performance of the top 10\% of models and optimal hyperparameter configurations per study and the relative importance of each hyperparameter based on PED-ANOVA in \cref{app:cifar-results}.

\section{Discussion} \label{sec:discussion}

In this work, we introduce the Surprisal-Rényi Free Energy (SRFE), a parametric divergence that interpolates between forward and reverse KL divergences while preserving the Fisher-Rao local geometry. Unlike classical $f$-divergences such as the Cressie-Read power divergence family, SRFE arises from a log-moment generating functional, making higher-order fluctuations of the log-likelihood ratio explicit and inducing a cumulant-based geometry.  This structure yields improved gradient conditioning under near-support mismatch and provides a tunable mechanism for controlling mass-covering versus mode-seeking behavior via a single parameter $\tau$. Together, these results position SRFE as a principled objective for robust and risk-sensitive generative modeling.

Empirical results on synthetic benchmarks in the continuous and discrete settings confirm that SRFE smoothly interpolates between the forward and reverse KL regimes. Notably, we demonstrate in a simple multi-class classification setting that SRFE provides the ability to tune prediction entropy and confidence as a function of $\tau$ under close-to-identical end-task performance.

In our evaluation on CIFAR-LT, we observed that SRFE matches or marginally outperforms standard CE baseline performance (test accuracy and ECE) on overall, many-shot, and medium-shot categories, and outperformed CE on the few-shot categories. Our study identified $\tau$ as a significant determinant of model performance (accounting for majority of the variance in validation accuracy via PED-ANOVA in half of the studies) (\cref{tab:hparam_importance}); however, the optimal choice of $\tau$ (in combination with other standard ML hyperparameters) as a function of the dataset's characteristics remains an open problem. We provide a theoretical analysis of the optimal label smoothing value for SRFE in \cref{app:label-smoothing}; however, we saw mixed results when comparing our theoretical expectations with the empirical results from the CIFAR-LT studies \pcref{app:cifar-results-label-smoothing}.

Our theoretical results motivate the evaluation of SRFE in domains that require careful tuning of distributional matching and sensitivity to tail-risk events. Applications that require robustness to corrupted labels (for example, the corrupted CIFAR and ImageNet datasets \citep{hendrycks2019robustness}) and large language model knowledge distillation stand out as promising directions for future work. Additionally, the explicit control over the learning dynamics provided through $\tau$ may be beneficial for learning paradigms like online and continual learning where changes in training data distributions require dynamic adjustments to the weighting of new and old knowledge.

\bibliography{references}
\bibliographystyle{tmlr}

\appendix
\crefalias{section}{appendix}
\crefalias{subsection}{appendix}
\crefalias{subsubsection}{appendix}

\section{Notation} \label{app:notation}

\cref{tab:notation} summarizes the notation used in our work.

\begin{table}[htbp]
\centering
\caption{Notation}
\label{tab:notation}
\begin{tabular}{@{}ll@{}}
\toprule
\textbf{Symbol}      & \textbf{Definition}                              \\
\midrule
$P, p$               & Target distribution                              \\
$Q_\theta, q_\theta$ & Approximate distribution / variational posterior \\
$\KL{P}{Q}$          & Forward KL divergence                            \\
$\KL{Q}{P}$          & Reverse KL divergence                            \\
$D_f(P\|Q)$          & $f$-divergence                                   \\
$\CR{P}{Q}$          & Cressie-Read power divergence family             \\
$D_\alpha(P\|Q)$     & R\'enyi divergence                               \\
$\SRFE{P}{Q}$        & Surprisal R\'enyi Free Energy (SRFE)             \\
$\tau$               & SRFE parameter                                   \\
$F(\tau)$            & Chernoff $\tau$-coefficient                      \\
$r_\tau$             & SRFE target escort distribution                  \\
$u(x)$              & Density ratio                                    \\
$\Delta(x)$          & Log-likelihood ratio                             \\
\bottomrule
\end{tabular}
\end{table}
\section{Related Work} \label{app:related-work}

The study of divergence measures precedes much of modern machine learning. Below, we provide additional context for our work on the Surprisal-R\'enyi Free Energy (SRFE).

\paragraph{Parametric Divergences}

A central objective in probabilistic modeling is the design of divergence measures that smoothly interpolate between operational regimes such as the mass-covering behavior of the forward KL divergence and the mode-seeking behavior of the reverse KL divergence. Classical examples include the Cressie-Read power divergence family \citep{cressie1984} and Amari's $\alpha$-divergences \citep{Amari1985}, which achieve this within the framework of $f$-divergences \citep{csiszar1963}. An $f$-divergence evaluates an expectation of the form $D_f(P\|Q) = \E_{x\sim Q}[f(u(x))]$, acting as a linear operator on the transformation of the likelihood ratio $p(x) / q(x)$. In this formulation, higher-order fluctuations of the log-likelihood ratio $\log \frac{p(x)}{q(x)}$ only enter through the shape of the generator $f$. Consequently, they optimize distribution matching strictly at the level of the first moment, failing to account for higher-order fluctuations, variance, or extreme tail behavior of the log-likelihood ratio.

For this, \cite{Renyi1961} introduced a non-linear parametric family with a logarithm outside of the expectation. Let $F(\alpha) = \int_{\mathcal{X}} p(x)^\alpha q(x)^{1-\alpha} \, d\mu(x)$ for some $\alpha>0$, also known as the Chernoff coefficient. Then, for $\alpha \in (0, 1) \cup (1, \infty)$, the R\'enyi divergence of order $\alpha$ is
\begin{equation*}
    D_\alpha(P \| Q) = \frac{1}{\alpha - 1} \log F(\alpha) = \frac{1}{\alpha - 1} \log \E_{x \sim Q} \left[ \left(p(x)/q(x)\right)^\alpha \right].
\end{equation*}
\cite{vanerven2014} establish several properties that make R\'enyi divergences particularly appealing for robust machine learning and free-energy formulations. To highlight some properties, $D_\alpha(P \| Q)$ is non-decreasing in $\alpha \in [0, \infty]$ \citep[Thm.~3]{vanerven2014}. At the minimum of this range, $D_0(P \| Q) = -\log Q(\{x : p(x) > 0\})$ only checks whether $Q$ places mass outside the support of $P$ \citep[Thm.~4]{vanerven2014}; at the top of the range, $D_\infty(P \| Q) = \log \operatorname{ess\,sup} (p/q)$ measures the worst-case under-coverage of $P$ by $Q$ \citep[Thm.~6]{vanerven2014}. Lastly, as $\alpha\to 1$, $D_\alpha(P \| Q) \to D_{\mathrm{KL}}(P \| Q)$ \citep[Thm.~5]{vanerven2014}.
 
The coefficient $F(\alpha)$, common to the R\'enyi divergence and SRFE, is known as the \emph{Chernoff $\alpha$-coefficient} \citep{nielsen2011}. The \emph{Chernoff information} (also referred to as the Chernoff divergence) is defined by minimizing this coefficient over $\alpha \in (0, 1)$:
\begin{equation*}
    D_{\mathrm{C}}(P \| Q) = -\log \min_{\alpha \in (0,1)} F(\alpha).
\end{equation*}
In classical hypothesis testing, $D_{\mathrm{C}}(P \| Q)$ is the best possible exponent of the Bayes error probability in discriminating between distributions $P$ and $Q$ \citep{Chernoff1952}. \cite{nielsen2011, nielsen2013} characterized the information-geometric structure of $D_{\mathrm{C}}$ on exponential family manifolds and showed that the optimal $\alpha^*$ selects the escort distribution $r_{\alpha^*}$, referred to as the \emph{Chernoff point}, at which the exponential geodesic $\{r_\alpha : \alpha \in (0, 1)\}$ joining $P$ and $Q$ meets the bisector $\{R : D_{\mathrm{KL}}(R \| P) = D_{\mathrm{KL}}(R \| Q)\}$.

\paragraph{Applications to Machine Learning}

The choice of divergence governing an objective heavily dictates the trained model's inductive bias. While much of ML relies on the forward KL divergence, standard divergences have well-documented tradeoffs that have motivated the use of lesser known alternatives.
Previous works have studied the utility of various divergences in the context of variational inference (VI). Notably, \cite{Li2016} extended VI to use R\'enyi's $\alpha$-divergence and unified many existing methods as special cases under their framework. Since then, other works have explored the use of other families of divergence in variational methods such as the skew-symmetric Jensen-Shannon divergence \citep{deasy2020} and Student's $t$-distribution \citep{kim2024}.
More recently, the tradeoffs between the forward and reverse KL divergences have become a major topic in LLMs, particularly in knowledge distillation. The nuances between the forward and reverse KL-based objectives are being examined more closely \citep{anonymous2025onllmknowledge} and several works have suggested approaches to address the limitations of subscribing to a single endpoint \citep{Wu2025,Gu2026,wen-etal-2023-f}.
In computer vision, the standard training paradigm relies on empirical risk minimization under the cross-entropy (CE) objective. CE is dominated by the majority classes and generalizes poorly to minority classes as the objective encourages models to prioritize learning representations for classes that occur more frequently. However, real-world data (and datasets created to more closely resemble real-world distributions such as ImageNet-LT \citep{openlongtailrecognition} and iNaturalist \citep{Horn_2018_CVPR}) frequently exhibit long-tailed class distributions. Although a variety of paradigms exist for long-tail learning, we focused on the design of risk-sensitive objectives that do not require additional steps like sample weighting or logit adjustment.
\section{Proofs} \label{app:proofs}

This section provides proofs of the theorems, lemmas, and corollaries presented in the paper. Each subsection begins with a copy of the statement followed by the proof.

\hypertarget{CR-KL-limits}{
\subsection{Proof of Theorem~\ref{thm:CR-KL-limits}} \label{app:CR-KL-limits}
}

\begin{shadedTheorem}
\thmcrkllimits*
\end{shadedTheorem}

For this proof, we consider the limits $\lambda\to 0$ and $\lambda\to -1$ individually. Each case follows the same three steps:
\begin{enumerate}
    \item Find Taylor expansion of $r(x)^\lambda$
    \item Substitute expansion into the expression for CR
    \item Evaluate expression at limit
\end{enumerate}

\begin{proof}
Let $r(x) \coloneq p(x)/q(x)$ and $\Delta(x) \coloneq \log r(x)$ and define the discrete CR based on \cref{eq:cr-divergence}:
\begin{equation}
\label{eq:cr-divergence-discrete}
\CRlambda{P}{Q}
=
\frac{1}{\lambda(\lambda+1)}
\sum_x
p(x)\!\left[
\left(\frac{p(x)}{q(x)}\right)^\lambda - 1
\right]
=
\frac{1}{\lambda(\lambda+1)}
\sum_x
p(x)\!\left[
r(x)^\lambda - 1
\right].
\end{equation}

\paragraph{Limit $\lambda\to 0$.}
Using the expansion $r(x)^\lambda = e^{\lambda\Delta(x)} = 1 + \lambda\Delta(x) + O(\lambda^2)$,
\[
r(x)^\lambda - 1 = \lambda\Delta(x) + O(\lambda^2).
\]
Substituting into the summation term of \cref{eq:cr-divergence-discrete},
\[
\sum_x p(x)\big(r(x)^\lambda - 1\big)
= \lambda\sum_x p(x)\Delta(x) + O(\lambda^2).
\]
The denominator satisfies $\lambda(\lambda+1)=\lambda + O(\lambda^2)$.
Therefore
\[
\CRlambda{P}{Q}
=
\frac{\lambda\sum_x p(x)\Delta(x) + O(\lambda^2)}
     {\lambda + O(\lambda^2)}
\;\xrightarrow[\lambda\to 0]{}\;
\sum_x p(x)\Delta(x)
=
\KL{P}{Q}.
\]

\paragraph{Limit $\lambda\to -1$.}
Write $\lambda = -1 + \mu$ with $\mu\to 0$.
Then
\[
r(x)^\lambda
= r(x)^{-1+\mu}
= r(x)^{-1} e^{\mu\Delta(x)}
= \frac{1}{r(x)}\Big(1 + \mu\Delta(x) + O(\mu^2)\Big).
\]
Thus
\[
r(x)^\lambda - 1
= \frac{1}{r(x)} - 1
+ \mu\frac{\Delta(x)}{r(x)}
+ O(\mu^2).
\]
Multiplying by $p(x)$,
\[
p(x)\big(r(x)^\lambda - 1\big)
=
p(x)\Big(\frac{1}{r(x)}-1\Big)
+ \mu\,p(x)\frac{\Delta(x)}{r(x)}
+ O(\mu^2).
\]
Since $p(x)/r(x) = q(x)$, this becomes
\[
p(x)\Big(\frac{1}{r(x)} - 1\Big) = q(x) - p(x), \qquad
p(x)\frac{\Delta(x)}{r(x)} = q(x)\Delta(x).
\]
Summing over $x$,
\[
\sum_x p(x)\big(r(x)^\lambda -1\big)
=
\sum_x (q(x)-p(x))
+ \mu\sum_x q(x)\Delta(x)
+ O(\mu^2).
\]
Because $\sum_x p(x)=\sum_x q(x)=1$, the first term is zero:
\[
\sum_x p(x)\big(r(x)^\lambda -1\big)
=
\mu\sum_x q(x)\Delta(x) + O(\mu^2).
\]
The denominator satisfies
\[
\lambda(\lambda+1)
= (-1+\mu)\mu
= -\mu + O(\mu^2).
\]
Hence
\[
\CRlambda{P}{Q}
=
\frac{\mu\sum_x q(x)\Delta(x) + O(\mu^2)}
     {-\mu + O(\mu^2)}
\;\xrightarrow[\mu\to 0]{}\;
-\sum_x q(x)\Delta(x)
=
\sum_x q(x)\log\frac{q(x)}{p(x)}
=
\KL{Q}{P}.
\]
\end{proof}

\hypertarget{SRFE-def}{
\subsection{Proof of Equation~\ref{eq:cr-associated}} \label{app:srfe-cr-associated}
}

\begin{shadedTheorem}[unbreakable]
\srfedef*
\end{shadedTheorem}

We provide the steps used to derive the associated CR divergence and show that it can be expressed using $F(\tau)$ as defined in \cref{eq:F-tau}. This allows for an easy comparison with the SRFE as presented in \cref{def:srfe}.

\begin{proof}
Substituting $\lambda = \tau - 1$ into the standard CR in \cref{eq:cr-divergence}, we get
\begin{align*}
    \CR{P}{Q}
    &=
    \frac{1}{\tau(\tau-1)}
    \int_\gX p(x) \left[\left(\dfrac{p(x)}{q(x)}\right)^{\tau-1} - 1\right] dx \\
    &=
    \frac{1}{\tau(\tau-1)}
    \int_\gX p(x)^\tau q(x)^{1-\tau} - p(x) dx \\
    &=
    \frac{1}{\tau(\tau-1)}
    \left[
    \int_\gX p(x)^\tau q(x)^{1-\tau} dx - \int_\gX p(x) dx
    \right] \\
    &=
    \frac{1}{\tau(\tau-1)}
    \left[
    \int_\gX p(x)^\tau q(x)^{1-\tau} dx - 1
    \right] \\
    &=
    \frac{1}{\tau(1 - \tau)}
    \left[
    1 - \int_\gX p(x)^\tau q(x)^{1-\tau} dx
    \right]
\end{align*}
Using the definition $F(\tau) = \int p(x)^\tau q(x)^{1-\tau}$, we get
\begin{equation*}
    \frac{1}{\tau(1 - \tau)}
    \left[
    1 - \int_\gX p(x)^\tau q(x)^{1-\tau} dx
    \right]
    =
    \frac{1 - F(\tau)}{\tau(1 - \tau)},
\end{equation*}
giving the form in \cref{eq:cr-associated}.
\end{proof}

\hypertarget{SRFE-KL-limits}{
\subsection{Proof of Theorem~\ref{thm:srfe-kl-limits}} \label{app:srfe-kl-limits}
}

\begin{shadedTheorem}
\srfelimits*
\end{shadedTheorem}

As with the proofs showing the limits of CR equal the forward and reverse KL divergences in \cref{app:CR-KL-limits}, we consider the limits as $\tau\to 1$ and $\tau\to 0$ for SRFE. 

\begin{proof}

For the following proofs, let $F(\tau) \coloneq \int_{\mathcal X} p(x)^\tau q(x)^{1-\tau}\,dx$ as in \cref{eq:F-tau}.

\paragraph{Limit $\tau\to 1$.}

We begin by differentiating $F(\tau)$ under the integral sign,
\begin{equation*}
F'(\tau)
= \int_{\mathcal X} p(x)^\tau q(x)^{1-\tau}
\log\frac{p(x)}{q(x)}\,dx.
\end{equation*}
At $\tau=1$, we have
\begin{equation*}
F(1) = \int_{\mathcal X} p(x)\,dx = 1,
\qquad
F'(1) = \int_{\mathcal X} p(x)\log\frac{p(x)}{q(x)}\,dx
= \KL{P}{Q}.
\end{equation*}
A first–order Taylor expansion of $F(\tau)$ around $\tau=1$ yields
\begin{align}
F(\tau)
&= F(1) + F'(1) (\tau - 1) + o(\tau - 1) \nonumber \\
&= 1 - (1 - \tau)\KL{P}{Q} + o(1 - \tau).
\label{eq:Z-tau-1}
\end{align}
Since $F(1)=1$, write $F(\tau)=1+\delta(\tau)$, where $\delta(\tau)\to 0$ as $\tau\to 1$.
From \cref{eq:Z-tau-1},
\begin{equation*}
\delta(\tau)=-(1-\tau)\KL{P}{Q}+o(1-\tau).
\end{equation*}
Using the Taylor series $\log(1+x)=x+o(x)$ as $x\to 0$,
\begin{align*}
\log F(\tau)
&= \log\bigl(1+\delta(\tau)\bigr) \nonumber \\
&= \delta(\tau)+o(\delta(\tau)) \nonumber \\
&= -(1-\tau)\KL{P}{Q}+o(1-\tau).
\label{eq:log-Z-tau-1}
\end{align*}
Substituting into \cref{eq:srfe} gives
\begin{align*}
\lim_{\tau\to 1} \SRFE{P}{Q}
&= \lim_{\tau\to 1}
-\,\frac{\log F(\tau)}{\tau(1 - \tau)} \\
&= \lim_{\tau\to 1}
-\,\frac{-(1-\tau)\KL{P}{Q} + o(1-\tau)}{\tau(1-\tau)} \\
&= \KL{P}{Q}.
\end{align*}

\paragraph{Limit $\tau\to 0$.}

The argument is analogous. One shows that $F(0)= \int_\gX q(x) dx = 1$ and
\[
F'(0)
= \int_{\mathcal X} q(x)\log\frac{p(x)}{q(x)}\,dx
= -\KL{Q}{P},
\]
so that the Taylor expansion of $F(\tau)$ around $\tau=0$ yields
\begin{align}
F(\tau)
&= F(0) + F'(0)(\tau) + o(\tau) \nonumber \\
&= 1 - \tau\KL{Q}{P} + o(\tau)
\label{eq:F-tau-0}
\end{align}
Since $F(0) = 1$, write $F(\tau) = 1 + \delta(\tau)$, where $\delta(\tau)\to0$ as $\tau\to 0$. From \cref{eq:F-tau-0},
\begin{equation*}
    \delta(\tau) = -\tau \KL{Q}{P} + o(\tau).
\end{equation*}
Using the Taylor series of $\log(1+x) = x + o(x)$ as $x\to 0$ again, we get
\begin{align}
    \log F(\tau)
    &= \log(1+\delta(\tau)) \nonumber \\
    &= \delta(\tau) + o(\delta(\tau)) \nonumber \\
    &= -\tau\KL{Q}{P}+o(\tau).
    \label{eq:log-F-tau-0}
\end{align}
Substituting \cref{eq:log-F-tau-0} into \cref{eq:srfe} yields
\begin{align*}
\lim_{\tau\to 0} \SRFE{P}{Q}
&= \lim_{\tau\to 0}
-\,\frac{\log F(\tau)}{\tau(1-\tau)} \\
&= \lim_{\tau\to 0}
-\,\frac{-\tau\KL{Q}{P}+o(\tau)}{\tau(1-\tau)} \\
&= \KL{Q}{P}.
\end{align*}
This shows that SRFE interpolates between forward and reverse KL.
\end{proof}

\hypertarget{SRFE-nonnegativity}{
\subsection{Proof of Lemma~\ref{lem:srfe-nonneg}} \label{app:srfe-nonneg}
}

\begin{shadedTheorem}
\srfenonneg*
\end{shadedTheorem}

We prove that the SRFE satisfies the property of nonnegativity for all $p,q$ with equality if and only if $p=q$.

\begin{proof}
Since $p,q$ are densities, $p^\tau q^{1-\tau}\ge 0$ and $F(\tau)\ge 0$.
Assume $p,q>0$ on a common support set of positive $\mu$-measure so that
$F(\tau)>0$ (otherwise SRFE is $+\infty$ by definition).

To prove the upper bound $F(\tau)\le 1$, apply H\"older's inequality with
exponents $a:=\frac{1}{\tau}>1$ and $b:=\frac{1}{1-\tau}>1$ so that
$\frac{1}{a}+\frac{1}{b}=\tau+(1-\tau)=1$. Write
\[
p^\tau q^{1-\tau} = (p)^\tau (q)^{1-\tau}.
\]
Then
\begin{align*}
F(\tau)
=\int p^\tau q^{1-\tau}\,d\mu
&\le
\left(\int (p^\tau)^a\,d\mu\right)^{1/a}
\left(\int (q^{1-\tau})^b\,d\mu\right)^{1/b}\\
&=
\left(\int p\,d\mu\right)^{\tau}
\left(\int q\,d\mu\right)^{1-\tau}
=
1^\tau\cdot 1^{1-\tau}
=1.
\end{align*}
This proves $0\le F(\tau) \le 1$.

For the equality condition: H\"older's inequality is tight if and only if
there exists a constant $c>0$ such that
\[
(p^\tau)^a = c\,(q^{1-\tau})^b
\quad \mu\text{-a.e. on the common support.}
\]
Since $(p^\tau)^a = p$ and $(q^{1-\tau})^b=q$, this becomes $p=cq$ $\mu$-a.e.
Integrating both sides gives $1=\int p\,d\mu = c\int q\,d\mu=c$, hence $c=1$
and therefore $p=q$ $\mu$-a.e. Conversely, if $p=q$ then $F(\tau)=\int p=1$.

Finally, since $F(\tau)\in(0,1]$, we have $\log F(\tau)\le 0$, and because
$\tau(1-\tau)>0$, it follows that
\[
\SRFE{P}{Q}
=
-\frac{1}{\tau(1-\tau)}\log F(\tau)
\ge 0,
\]
with equality if and only if $F(\tau)=1$, i.e. iff $p=q$ $\mu$-a.e.
\end{proof}

\hypertarget{SRFE-monotone}{
\subsection{Proof of Lemma~\ref{lem:srfe-cr-monotone}} \label{app:srfe-cr-monotone}
}

\begin{shadedTheorem}
\monotone*
\end{shadedTheorem}

\begin{proof}
Rearranging \cref{eq:cr-associated}, we get
\[
\CR{P}{Q}
= \frac{1 - F(\tau)}{\tau(1-\tau)}
\quad\Longrightarrow\quad
F(\tau)
= 1 - \tau(1-\tau)\,\CR{P}{Q}.
\]
Substituting this into the SRFE expression gives
\[
\SRFE{P}{Q}
=
-\frac{1}{\tau(1-\tau)}
\log\bigl(1-\tau(1-\tau)
\CR{P}{Q}
\bigr).
\]
Let
\begin{equation*}
    h_\tau(d) \coloneq -\dfrac{1}{\tau(1-\tau)}\log(1-\tau(1-\tau)d),
\end{equation*}
for $d\in\left[0,\,\tfrac{1}{\tau(1-\tau)}\right)$, then
\[
\SRFE{P}{Q}
= h_\tau(\CR{P}{Q}).
\]

Since $P$ and $Q$ are mutually absolutely continuous, we have
$F(\tau) \in (0,1]$, hence
\[
\tau(1-\tau)\CR{P}{Q}
= 1 - F(\tau)
\in [0,1),
\]
so $h_\tau$ is well-defined on $d\in\left[0,\,\tfrac{1}{\tau(1-\tau)}\right)$.
Differentiating $h_\tau(d)$ yields
\[
h_\tau'(d)
= -\frac{1}{\tau(1-\tau)}
  \cdot \frac{-\tau(1-\tau)}{1-\tau(1-\tau)d}
= \frac{1}{1-\tau(1-\tau)d} > 0,
\]
showing that $h_\tau$ is smooth and strictly increasing on its domain $d$.

Strict monotonicity implies
\[
\CR{P}{Q_1} < \CR{P}{Q_2}
\quad\Longleftrightarrow\quad
h_\tau\big(\CR{P}{Q_1}\big)
<
h_\tau\big(\CR{P}{Q_2}\big),
\]
which gives the stated ordering for SRFE.

Finally, $F(\tau)=1$ if and only if $p=q$ a.e., which gives
\[
\CR{P}{P} = 0
\qquad\text{and}\qquad
\SRFE{P}{P} = 0.
\]
Since $F(\tau)\le 1$, these are the global minimum values over $P$.
This establishes the uniqueness of the common minimizer and completes the proof.
\end{proof}

\hypertarget{SRFE-not-f-div}{
\subsection{Proof of Theorem~\ref{thm:srfe-not-fdiv}} \label{app:srfe-not-fdiv}
}

\begin{shadedTheorem}
\srfenotfdiv*
\end{shadedTheorem}

To prove that SRFE is \emph{not} an $f$-divergence, we use a proof by contradiction. We consider the simple case of a 2-simplex $\Delta^2 = \{(p_0,p_1,p_2)\in\R^3 | \sum_{i=0}^2 p_i=1, p_i\ge 0 \text{ for } i=0,1,2\}$ and assume that SRFE can be expressed as an $f$-divergence. We then compare its second derivative to the second derivative of the closed form SRFE on the 2-simplex case.

\begin{proof}
Assume for contradiction that there exists a generator $f$ such that
\begin{equation}
\label{eq:assume-fdiv}
\SRFE{P}{Q}=\sum_{i=1}^3 q_i\, f\!\left(\frac{p_i}{q_i}\right)
\quad\text{for all } p,q\in\Delta^2.
\end{equation}

Fix $q=(1/3,1/3,1/3)$ and parameterize $p$ by two free variables:
for $(u,v)$ in an open set where $u>0$, $v>0$, and $u+v<1$, set
\[
p(u,v)=(u,v,1-u-v).
\]
Define the scalar function
\[
G(u,v):=D^{\mathrm{SRFE}}_\tau(p(u,v)\|q).
\]
With uniform $q$, SRFE becomes
\[
G(u,v)
=
-\frac{1}{\tau(1-\tau)}
\log\!\Big(3^{\tau-1}\big[u^\tau+v^\tau+(1-u-v)^\tau\big]\Big)
=
C_\tau
-\frac{1}{\tau(1-\tau)}\log S(u,v),
\]
where $C_\tau:=-\frac{\tau-1}{\tau(1-\tau)}\log 3$ is constant and
\[
S(u,v):=u^\tau+v^\tau+(1-u-v)^\tau.
\]

On the other hand, the assumed $f$-divergence representation \cref{eq:assume-fdiv}
with $q_i=1/3$ gives
\[
G(u,v)
=
\frac13 f(3u)+\frac13 f(3v)+\frac13 f(3(1-u-v)).
\]
Differentiate twice with respect to $u$ and $v$. Since the first two terms depend
only on $u$ or only on $v$, the \emph{mixed} second derivative comes only from
the third term:
\begin{equation}
\label{eq:mixed-fdiv}
\frac{\partial^2 G}{\partial u\,\partial v}(u,v)
=
\frac13 \cdot f''\!\big(3(1-u-v)\big)\cdot (-3)(-3)
=
3\, f''\!\big(3(1-u-v)\big).
\end{equation}
In particular, the mixed derivative depends \emph{only on $1-u-v$}.

Now compute the mixed derivative from the SRFE closed form.
Since $G(u,v)=C_\tau-\frac{1}{\tau(1-\tau)}\log S(u,v)$, we have
\[
\frac{\partial^2 G}{\partial u\,\partial v}
=
-\frac{1}{\tau(1-\tau)}
\left(
\frac{S_{uv}}{S}
-
\frac{S_u S_v}{S^2}
\right).
\]
A direct computation gives, for $\tau\in(0,1)$,
\[
S_u=\tau u^{\tau-1}-\tau(1-u-v)^{\tau-1},\quad
S_v=\tau v^{\tau-1}-\tau(1-u-v)^{\tau-1},
\]
and
\[
S_{uv}=\tau(\tau-1)(1-u-v)^{\tau-2}.
\]
Substituting yields
\begin{equation}
\label{eq:mixed-srfe}
\frac{\partial^2 G}{\partial u\,\partial v}(u,v)
=
-\frac{1}{\tau(1-\tau)}
\left(
\frac{\tau(\tau-1)w^{\tau-2}}{S(u,v)}
-
\frac{(\tau u^{\tau-1}-\tau w^{\tau-1})(\tau v^{\tau-1}-\tau w^{\tau-1})}{S(u,v)^2}
\right),
\end{equation}
where $w:=1-u-v$.

Crucially, the right-hand side of \cref{eq:mixed-srfe} depends on $u$ and $v$
\emph{separately} through the factors $u^{\tau-1}$ and $v^{\tau-1}$, not only
through $w=1-u-v$. Therefore, $\frac{\partial^2 G}{\partial u\,\partial v}(u,v)$
cannot be a function of $w$ alone on any open set.

This contradicts \cref{eq:mixed-fdiv}, which implies the mixed derivative must
be of the form $H(1-u-v) = H(w)$ for some one-variable function $H$.

Hence no such $f$ exists and SRFE cannot be an $f$-divergence.
\end{proof}

\hypertarget{CR-surprisal-variance}{
\subsection{Proof of Theorem~\ref{thm:CR-surprisal-variance}} \label{app:CR-surprisal-variance}
}

\begin{shadedTheorem}
\crsurprisal*
\end{shadedTheorem}

\begin{proof}
Let $r(x) \coloneq p(x) / q(x) = e^{\Delta(x)}$. Then
\[
\CR{P}{Q}
=
\frac{1}{\tau(\tau+1)}
\int p(x)\bigl(r(x)^\tau-1\bigr)\,d\mu(x)
=
\frac{1}{\tau(\tau+1)}
\int p(x)\bigl(e^{\tau\Delta(x)}-1\bigr)\,d\mu(x).
\]

\paragraph{Step 1: Taylor expansion of the numerator.}
By Taylor's theorem,
\[
e^{\tau\Delta(x)}
=
1+\tau\Delta(x)+\frac{\tau^2}{2}\Delta(x)^2 + R_3(\tau,x),
\]
where the remainder satisfies
\[
|R_3(\tau,x)|
\le
\frac{|\tau|^3}{6}|\Delta(x)|^3 e^{|\tau||\Delta(x)|}.
\]
Under $\E_p[|\Delta|^3]<\infty$, for $\tau$ in a neighborhood of $0$ the bound
is $p$-integrable, so integrating termwise yields
\[
\int p(x)\bigl(e^{\tau\Delta(x)}-1\bigr)\,d\mu(x)
=
\tau\,\E_p[\Delta]
+\frac{\tau^2}{2}\E_p[\Delta^2]
+O(\tau^3).
\]

\paragraph{Step 2: Divide by $\tau(\tau+1)$.}
Using $\tau(\tau+1)=\tau(1+\tau)$,
\begin{align*}
\CR{P}{Q}
&=
\frac{\tau\,\E_p[\Delta]+\frac{\tau^2}{2}\E_p[\Delta^2]+O(\tau^3)}
     {\tau(1+\tau)}\\
&=
\frac{\E_p[\Delta]+\frac{\tau}{2}\E_p[\Delta^2]+O(\tau^2)}
     {1+\tau}.
\end{align*}
Expanding $(1+\tau)^{-1}=1-\tau+O(\tau^2)$ (valid as $|\tau|<1$) gives
\[
\CR{P}{Q}
=
\E_p[\Delta]
+\tau\Bigl(-\E_p[\Delta]+\tfrac12\E_p[\Delta^2]\Bigr)
+O(\tau^2).
\]

\paragraph{Step 3: Express in terms of KL and variance.}
Since $\E_p[\Delta]=\KL{P}{Q}$ and
$\E_p[\Delta^2]=\Var_p[\Delta]+\E_p[\Delta]^2
=\Var_p[\Delta]+\KL{P}{Q}^2$, we obtain
\begin{align*}
\CR{P}{Q}
&=
\KL{P}{Q}
+\tau\Bigl(-\KL{P}{Q}+\tfrac12\Var_p[\Delta]+\tfrac12\KL{P}{Q}^2\Bigr)
+O(\tau^2)\\
&=
\KL{P}{Q}
+\frac{\tau}{2}\Var_p[\Delta]
+\tau\Bigl(\tfrac12\KL{P}{Q}^2-\KL{P}{Q}\Bigr)
+O(\tau^2),
\end{align*}
which is \cref{eq:CR-surprisal-variance}.
\end{proof}

\hypertarget{SRFE-surprisal-forward}{
\subsection{Proof of Theorem~\ref{thm:srfe-local-forward-KL}} \label{app:srfe-local-forward-KL}
}

\begin{shadedTheorem}
\srfeforwardexpansion*
\end{shadedTheorem}

\begin{proof}
Let
$F(\tau) := \int p(x)^\tau q(x)^{1-\tau}\,dx$
and
$Z(\tau) := \log F(\tau)$.
A second-order Taylor expansion of $Z(\tau)$ at $\tau=1$ gives
\begin{equation}
\label{eq:Z-Taylor-at-1}
Z(\tau)
=
Z(1) + (\tau-1)Z'(1)
+ \frac{1}{2}(\tau-1)^2 Z''(1)
+ O\bigl((\tau-1)^3\bigr).
\end{equation}

We first compute $Z'(\tau)$ and $Z''(\tau)$:
\[
Z'(\tau)
= \frac{F'(\tau)}{F(\tau)},
\qquad
Z''(\tau)
= \frac{F(\tau)F''(\tau) - F'(\tau)^2}{F(\tau)^2}.
\]

Differentiating $F(\tau)$ under the
integral sign,
\begin{align}
F'(\tau)
&= \int \Bigl(p(x)^\tau \log p(x)\Bigr) q(x)^{1-\tau}\,dx
   + \int p(x)^\tau \Bigl(q(x)^{1-\tau}(-\log q(x))\Bigr)\,dx \nonumber \\
&= \int p(x)^\tau q(x)^{1-\tau}
   \bigl(\log p(x) - \log q(x)\bigr)\,dx \nonumber \\
&= \int p(x)^\tau q(x)^{1-\tau} \Delta(x)\,dx,
\label{eq:F-prime-tau}
\\[0.4em]
F''(\tau)
&= \int p(x)^\tau q(x)^{1-\tau} \Delta(x)^2\,dx.
\label{eq:F-double-prime-tau}
\end{align}

At $\tau=1$ we have $F(1) = \int p(x)\,dx = 1$, and
\begin{align*}
Z(1)
&= \log F(1) = 0,\\
Z'(1)
&= F'(1)
 = \int p(x)\Delta(x)\,dx
 = \int p(x)\log\frac{p(x)}{q(x)}\,dx
 = \KL{P}{Q},\\[0.3em]
Z''(1)
&= F''(1) - F'(1)^2\\
&= \int p(x)\Delta(x)^2\,dx
   - \Bigl(\int p(x)\Delta(x)\,dx\Bigr)^2\\
&= \Var_p[\Delta].
\end{align*}

Substituting into \cref{eq:Z-Taylor-at-1} gives
\begin{equation*}
\label{eq:F-expansion-at-1}
Z(\tau)
=
(\tau-1)\KL{P}{Q}
+
\frac{(\tau-1)^2}{2}\Var_p[\Delta]
+
O\bigl((\tau-1)^3\bigr).
\end{equation*}

Factoring out $(\tau-1)$, we get
\[
Z(\tau)
=
(\tau-1)
\Bigl(
\KL{P}{Q}
+
\frac{\tau-1}{2}\Var_p[\Delta]
+
O\bigl((\tau-1)^2\bigr)
\Bigr).
\]
Since $\tau(1-\tau) = -\tau(\tau-1)$, we obtain
\begin{align*}
\SRFE{P}{Q}
&=
-\frac{Z(\tau)}{\tau(1-\tau)}
=
-\frac{\tau-1}{\tau(1-\tau)}
\Bigl(
\KL{P}{Q}
+
\frac{\tau-1}{2}\Var_p[\Delta]
+
O\bigl((\tau-1)^2\bigr)
\Bigr).
\end{align*}
Note that
\[
-\frac{\tau-1}{\tau(1-\tau)}
=
\frac{1}{\tau}\cdot\frac{1-\tau}{1-\tau}
=
\frac{1}{\tau},
\]
so
\[
\SRFE{P}{Q}
=
\frac{1}{\tau}
\Bigl(
\KL{P}{Q}
+
\frac{\tau-1}{2}\Var_p[\Delta]
+
O\bigl((\tau-1)^2\bigr)
\Bigr).
\]
Now expand $1/\tau$ around $\tau=1$:
\[
\frac{1}{\tau}
=
\frac{1}{1+(\tau-1)}
=
1 - (\tau-1) + O\bigl((\tau-1)^2\bigr).
\]
Thus
\begin{align*}
\SRFE{P}{Q}
&= \Bigl(1 - (\tau-1) + O\bigl((\tau-1)^2\bigr)\Bigr)
\Bigl(\KL{P}{Q}
      + \frac{\tau-1}{2}\Var_p[\Delta]
      + O\bigl((\tau-1)^2\bigr)\Bigr)\\
&= \KL{P}{Q}
+ (\tau-1)\Bigl(-\KL{P}{Q} + \tfrac12\Var_p[\Delta]\Bigr)
+ O\bigl((\tau-1)^2\bigr), \\
&= \KL{P}{Q}
+ (1-\tau)\Bigl(\KL{P}{Q} - \tfrac12\Var_p[\Delta]\Bigr)
+ O\bigl((\tau-1)^2\bigr),
\end{align*}
which is exactly \cref{eq:srfe-taylor-forward}.
\end{proof}

\hypertarget{SRFE-surprisal-backward}{
\subsection{Proof of Theorem~\ref{thm:srfe-local-backward-KL}} \label{app:srfe-local-backward-KL}
}

\begin{shadedTheorem}[unbreakable]
\srfebackwardexpansion*
\end{shadedTheorem}

\begin{proof}
Define $F(\tau)$ and $Z(\tau)$ as in \cref{app:srfe-local-forward-KL}. A second-order Taylor expansion of $Z$ at $\tau=0$ gives
\begin{equation}
\label{eq:Z-Taylor-at-0}
Z(\tau)
=
Z(0) + \tau Z'(0)
+ \frac{1}{2}\tau^2 Z''(0)
+ O\bigl(\tau^3\bigr).
\end{equation}

Using \Cref{eq:F-prime-tau,eq:F-double-prime-tau}, we get that at $\tau=0$, $F(0) = \int q(x) dx = 1$, and
\begin{align*}
    Z(0) &= \log F(0) = 0,\\
    Z'(0) &= F'(0) = \int q(x) \Delta(x) dx = \int q(x) \log\dfrac{p(x)}{q(x)} = -\KL{Q}{P}, \\
    Z''(0) &= F''(0) - F'(0)^2 \\
    &= \int q(x)\Delta(x)^2 dx - \left(\int q(x) \Delta(x) \right)^2 \\
    &= \Var_q [\Delta]
\end{align*}
Substituting into \Cref{eq:Z-Taylor-at-0} gives
\begin{equation*}
\label{eq:Z-expansion-at-0}
Z(\tau)
=
-\tau\KL{Q}{P}
+
\frac{\tau^2}{2}\Var_q[\Delta]
+
O\bigl(\tau^3\bigr).
\end{equation*}
Factoring out $\tau$, we get
\[
Z(\tau)
=
\tau
\Bigl(
-\KL{Q}{P}
+
\frac\tau2 \Var_q[\Delta]
+
O\bigl(\tau^2\bigr)
\Bigr).
\]

Then,
\begin{align*}
\SRFE{P}{Q}
&= -\frac{Z(\tau)}{\tau(1-\tau)} \nonumber \\
&= -\frac{\tau}{\tau(1-\tau)}
\Bigl(
-\KL{Q}{P}
+
\frac\tau2 \Var_q[\Delta]
+
O\bigl(\tau^2\bigr)
\Bigr) \nonumber \\
&= \frac{1}{1-\tau}
\Bigl(
\KL{Q}{P}
-
\frac{\tau}{2}\Var_q[\Delta]
+
O\bigl(\tau^2\bigr)
\Bigr)
\end{align*}

Now expand $1/(1-\tau)$ around $\tau=0$:
\[
\frac{1}{1-\tau} = 1 + \tau + \gO(\tau^2).
\]
Thus
\begin{align*}
\SRFE{P}{Q}
&=
\Bigl(1 + \tau + O\bigl(\tau^2\bigr)\Bigr)
\Bigl(\KL{Q}{P}
      - \frac{\tau}{2}\Var_q[\Delta]
      + O\bigl(\tau^2\bigr)\Bigr)\\[0.3em]
&=
\KL{Q}{P}
+ \tau\Bigl(\KL{Q}{P} - \tfrac12\Var_q[\Delta]\Bigr)
+ O\bigl(\tau^2\bigr),
\end{align*}
which is exactly \cref{eq:srfe-taylor-backward}.
\end{proof}

\hypertarget{SRFE-gradient}{
\subsection{Proof of Lemma~\ref{lem:srfe-gradient}} \label{app:srfe-gradient}
}

\begin{shadedTheorem}[unbreakable]
\srfegrad*
\end{shadedTheorem}

\begin{proof}
Using $\nabla \log x = (\nabla x)/x$ and $\SRFE{P}{Q_\theta} = -\dfrac{1}{\tau(1-\tau)}\log F(\tau)$,
\[
\nabla_\theta \SRFE{P}{Q_\theta}
=
-\frac{1}{\tau(1-\tau)}\,\frac{\nabla_\theta F(\tau)}{F(\tau)}.
\]
Next,
\begin{align*}
\nabla_\theta F(\tau)
&= \nabla_\theta \int_{\mathcal X} p(x)^\tau q_\theta(x)^{1-\tau}\,d\mu(x) \\
&= \int_{\mathcal X} p(x)^\tau \nabla_\theta\!\left[q_\theta(x)^{1-\tau}\right]\,d\mu(x) \\
&= \int_{\mathcal X} p(x)^\tau (1-\tau) q_\theta(x)^{-\tau} \nabla_\theta q_\theta(x)\,d\mu(x).
\end{align*}
Substituting in $\nabla x = x \cdot \nabla \log x$,
\begin{align}
\nabla_\theta F(\tau)
&= \int_{\mathcal X} p(x)^\tau (1-\tau) q_\theta(x)^{-\tau} \cdot q_\theta(x) \nabla_\theta \log q_\theta(x)\,d\mu(x) \nonumber \\
&= (1-\tau)\int_{\mathcal X} p(x)^\tau q_\theta(x)^{1-\tau} \nabla_\theta \log q_\theta(x)\,d\mu(x). \label{eq:grad-F-tau}
\end{align}
Substituting into $\nabla_\theta \SRFE{P}{Q_\theta}$,
\begin{equation*}
\nabla_\theta \SRFE{P}{Q_\theta}
=
-\frac{1}{\tau}\int_{\mathcal X} \frac{p(x)^\tau q_\theta(x)^{1-\tau}}{F(\tau)}\,\nabla_\theta \log q_\theta(x)\,d\mu(x)
=
-\frac{1}{\tau}\E_{x\sim r_\tau}\!\left[\nabla_\theta \log q_\theta(x)\right].
\end{equation*}
\end{proof}

\hypertarget{CR-gradient}{
\subsection{Proof of Lemma~\ref{lem:cr-gradient}} \label{app:cr-gradient}
}

\begin{shadedTheorem}
\crgrad*
\end{shadedTheorem}

\begin{proof}
By definition, $\CR{P}{Q_\theta}=\dfrac{1-F(\tau)}{\tau(1-\tau)}$, hence
\[
\nabla_\theta \CR{P}{Q_\theta}
=
-\frac{1}{\tau(1-\tau)}\,\nabla_\theta F(\tau).
\]
Substituting in \cref{eq:grad-F-tau} and simplifying,
\[
\nabla_\theta \CR{P}{Q_\theta}
=
-\frac{1}{\tau}\int_{\mathcal X} p(x)^\tau q_\theta(x)^{1-\tau}\nabla_\theta \log q_\theta(x)\,d\mu(x).
\]
Since $p(x)^\tau q_\theta(x)^{1-\tau}=q_\theta(x)\left(\frac{p(x)}{q_\theta(x)}\right)^\tau
= q_\theta(x)u(x)^\tau$, we obtain
\[
\nabla_\theta \CR{P}{Q_\theta}
=
-\frac{1}{\tau}\int_{\mathcal X} q_\theta(x) u(x)^\tau \nabla_\theta \log q_\theta(x)\,d\mu(x)
=
-\frac{1}{\tau}\E_{x\sim Q_\theta}\!\left[u(x)^\tau \nabla_\theta \log q_\theta(x)\right],
\]
which is \cref{eq:cr-gradient}. Finally, using $\E_{x\sim Q_\theta}[\nabla_\theta \log q_\theta(x)]
=\nabla_\theta \int q_\theta(x)\,d\mu(x)=\nabla_\theta 1 = 0$,
we can subtract the constant baseline $1$ to get the equivalent form with $(u^\tau-1)$ as in \cref{eq:cr-gradient-subtracted}.
\end{proof}

\hypertarget{CR-gradient-estimator}{
\subsection{Proof of Lemma~\ref{lem:cr-second-moment}} \label{app:cr-second-moment}
}

\begin{shadedTheorem}
\crsecondmoment*
\end{shadedTheorem}

The variance of the expectation of the gradient is trivially zero; however, we consider the variance of the stochastic gradient for insight into the learning dynamics of each objective.

\begin{proof}
From \cref{eq:cr-gradient}, define the stochastic gradient of CR as $g_{\mathrm{CR}}(x) = -(1/\tau) u(x)^\tau \nabla_\theta \log q_\theta(x)$ for $x\sim Q_\theta$ and assume
$\|\nabla_\theta \log q_\theta(x)\|^2 \le C$.
\begin{equation*}
\|g_{\mathrm{CR}}(x)\|^2
=
\frac{1}{\tau^2}\,u(x)^{2\tau}\,\|\nabla_\theta\log q_\theta(x)\|^2
\le
\frac{C}{\tau^2}\,u(x)^{2\tau}.
\end{equation*}
The identity
$\mathbb{E}_{Q_\theta}[u^{2\tau}]
=\int q_\theta u^{2\tau}\,d\mu
$
is immediate. Therefore
\begin{equation*}
\E_{x\sim Q_\theta} \left[\|g_{\mathrm{CR}}(x)\|^2\right]
\le
\frac{C}{\tau^2} \int q_\theta(x) u(x)^{2\tau} d\mu(x),
\end{equation*}
which gives us \cref{eq:cr-second-moment}.
\end{proof}

\hypertarget{SRFE-gradient-estimator}{
\subsection{Proof of Lemma~\ref{lem:srfe-second-moment}} \label{app:srfe-second-moment}
}

\begin{shadedTheorem}
\srfesecondmoment*
\end{shadedTheorem}

\begin{proof}
We use the SRFE gradient identity (proved in \cref{lem:srfe-gradient}):
\begin{equation}
\label{eq:srfe-grad-escort}
\nabla_\theta \SRFE{P}{Q_\theta}
=
-\frac{1}{\tau}\,
\E_{X\sim r_\tau}\!\left[\nabla_\theta \log q_\theta(X)\right],
\qquad
r_\tau(x)=\frac{p(x)^\tau q_\theta(x)^{1-\tau}}{F(\tau)}.
\end{equation}

\smallskip
\noindent
\textbf{(i) Unbiasedness under escort sampling.}
By definition from \cref{eq:srfe-escort-estimator},
\[
\mathbb{E}_{X\sim r_\tau}\!\left[g_{\mathrm{SRFE}}^{(r)}(X)\right]
=
-\frac{1}{\tau}\,
\mathbb{E}_{X\sim r_\tau}\!\left[\nabla_\theta \log q_\theta(X)\right]
=
\nabla_\theta \SRFE{P}{Q_\theta},
\]
which proves unbiasedness.

\smallskip
\noindent
\textbf{(i) Second moment bound.}
Using $\|g_{\mathrm{SRFE}}^{(r)}(X)\|^2= (1/\tau^2) \|\nabla_\theta \log q_\theta(X)\|^2$ and the uniform bound
$\|\nabla_\theta \log q_\theta(x)\|^2\le C$,
\[
\mathbb{E}_{X\sim r_\tau}\!\big[\|g_{\mathrm{SRFE}}^{(r)}(X)\|^2\big]
=
\frac{1}{\tau^2}\mathbb{E}_{X\sim r_\tau}\!\big[\|\nabla_\theta \log q_\theta(X)\|^2\big]
\le
\frac{C}{\tau^2},
\]
which is \cref{eq:srfe-escort-second-moment}.

\smallskip
\noindent
\textbf{(ii) Unbiasedness under $Q_\theta$-sampling (importance form).}
First note that for any integrable test function $\varphi$,
\[
\mathbb{E}_{X\sim r_\tau}[\varphi(X)]
=
\int_{\mathcal X} r_\tau(x)\varphi(x)\,d\mu(x)
=
\frac{1}{F(\tau)}\int_{\mathcal X} p(x)^\tau q_\theta(x)^{1-\tau}\varphi(x)\,d\mu(x).
\]
Since $p(x)^\tau q_\theta(x)^{1-\tau}=q_\theta(x)\big(\frac{p(x)}{q_\theta(x)}\big)^\tau
=q_\theta(x)u(x)^\tau$, this becomes
\begin{equation}
\label{eq:escort-importance-identity}
\mathbb{E}_{X\sim r_\tau}[\varphi(X)]
=
\frac{1}{F(\tau)}\int_{\mathcal X} q(x) u(x)^\tau\varphi(x)\,d\mu(x)
=
\frac{1}{F(\tau)}\mathbb{E}_{X\sim Q_\theta}\!\big[u(X)^\tau \varphi(X)\big].
\end{equation}
Apply \cref{eq:escort-importance-identity} with $\varphi(X)=\nabla_\theta\log q_\theta(X)$ to get 
\begin{equation}
\mathbb{E}_{X\sim r_\tau}[\nabla_\theta \log q_\theta(X)]
=
\frac{1}{F(\tau)}\mathbb{E}_{X\sim Q_\theta}\!\big[u(X)^\tau \nabla_\theta \log q_\theta(X) \big],
\end{equation}
and substitute into
\cref{eq:srfe-grad-escort}:
\[
\nabla_\theta \SRFE{P}{Q_\theta}
=
-\frac{1}{\tau}\cdot \frac{1}{F(\tau)}
\mathbb{E}_{X\sim Q_\theta}\!\big[u(X)^\tau \nabla_\theta\log q_\theta(X)\big]
=
\mathbb{E}_{X\sim Q_\theta}\!\left[g_{\mathrm{SRFE}}^{(q)}(X)\right],
\]
which proves unbiasedness of \cref{eq:srfe-q-estimator}.

\smallskip
\noindent
\textbf{(ii) Second moment bound.}
By \cref{eq:srfe-q-estimator} and the uniform score bound,
\[
\|g_{\mathrm{SRFE}}^{(q)}(X)\|^2
=
\frac{1}{\tau^2 F(\tau)^2}u(X)^{2\tau}\|\nabla_\theta\log q_\theta(X)\|^2
\le
\frac{C}{\tau^2 F(\tau)^2}u(X)^{2\tau}.
\]
Taking expectations under $X\sim Q_\theta$ yields
\[
\mathbb{E}_{X\sim Q_\theta}\!\big[\|g_{\mathrm{SRFE}}^{(q)}(X)\|^2\big]
\le
\frac{C}{\tau^2 F(\tau)^2}\mathbb{E}_{X\sim Q_\theta}[u(X)^{2\tau}].
\]
Finally,
\[
\mathbb{E}_{X\sim Q_\theta}[u(X)^{2\tau}]
=
\int_{\mathcal X} q_\theta(x)u(x)^{2\tau}\,d\mu(x),
\]
which gives \cref{eq:srfe-q-second-moment}.
\end{proof}

\hypertarget{SRFE-variational}{
\subsection{Proof of Theorem~\ref{thm:srfe-variational}} \label{app:srfe-variational}
}

\begin{shadedTheorem}[unbreakable]
\srfevariational*
\end{shadedTheorem}

\begin{proof}
Fix $\tau\in(0,1)$ and define
\[
Z_\tau:=\int_{\mathcal X} p(x)^\tau q(x)^{1-\tau}\,d\mu(x)\in(0,\infty),
\qquad
r_\tau(x):=\frac{p(x)^\tau q(x)^{1-\tau}}{Z_\tau}.
\]
For any $r\in\mathcal P(\mathcal X)$ with $r\ll p$ and $r\ll q$, expand the KL terms:
\[
\KL{r}{q}=\int r\log\frac{r}{q}\,d\mu,\qquad
\KL{r}{p}=\int r\log\frac{r}{p}\,d\mu.
\]
Consider the functional
\[
J(r):=\frac{1}{\tau}\KL{r}{q}+\frac{1}{1-\tau}\KL{r}{p}.
\]
Combine the two integrals:
\begin{align*}
J(r)
&=
\int r\left[\frac{1}{\tau}\log\frac{r}{q}+\frac{1}{1-\tau}\log\frac{r}{p}\right]d\mu\\
&=
\int r\left[\Big(\frac{1}{\tau}+\frac{1}{1-\tau}\Big)\log r
-\frac{1}{\tau}\log q-\frac{1}{1-\tau}\log p\right]d\mu\\
&=
\frac{1}{\tau(1-\tau)}
\int r\left[\log r-\tau\log p-(1-\tau)\log q\right]d\mu\\
&=
\frac{1}{\tau(1-\tau)}
\int r\log\frac{r}{p^\tau q^{1-\tau}}\,d\mu.
\end{align*}
Using $p^\tau q^{1-\tau}=Z_\tau\,r_\tau$, we obtain
\begin{align*}
\int r\log\frac{r}{p^\tau q^{1-\tau}}\,d\mu
&=
\int r\log\frac{r}{Z_\tau r_\tau}\,d\mu \\
&=
\int r\left[\log\frac{r}{r_\tau} - \log Z_\tau \right]\,d\mu \\
&=
\int r\log\frac{r}{r_\tau}\,d\mu - \log Z_\tau \int r\,d\mu \\
&=
\int r\log\frac{r}{r_\tau}\,d\mu - \log Z_\tau\\
&=
\KL{r}{r_\tau}-\log Z_\tau.
\end{align*}
Therefore,
\begin{equation}
\label{eq:J-decomp}
J(r)=\frac{1}{\tau(1-\tau)}\KL{r}{r_\tau}-\frac{1}{\tau(1-\tau)}\log Z_\tau.
\end{equation}
By nonnegativity of KL divergence, $\KL{r}{r_\tau}\ge 0$ with equality iff
$r=r_\tau$ a.e. Hence, $J(r)$ is minimized uniquely at $r_\tau$ and
\[
\min_r J(r)
=
-\frac{1}{\tau(1-\tau)}\log Z_\tau.
\]
Comparing with \cref{eq:srfe} yields \cref{eq:srfe-variational} and the minimizer \cref{eq:escort}.

Finally, rearranging \cref{eq:J-decomp} gives the Pythagorean decomposition \cref{eq:pythagorean}.
\end{proof}

\hypertarget{SRFE-riemannian}{
\subsection{Proof of Theorem~\ref{thm:srfe-metric}} \label{app:srfe-metric}
}

\begin{shadedTheorem}
\srferiemannian*
\end{shadedTheorem}

\begin{proof}
Fix $\theta$ and write $\ell_\theta(x):=\log p_\theta(x)$.
Let $\theta'=\theta+\delta$ with $\delta\in\mathbb R^d$ small, and define
\[
A(\delta):=\int p_\theta(x)^\tau\,p_{\theta+\delta}(x)^{1-\tau}\,d\mu(x).
\]
Since $p_{\theta+\delta}=p_\theta\exp(\ell_{\theta+\delta}-\ell_\theta)$, we get
\begin{align*}
A(\delta)
&=
\int p_\theta(x)^\tau\,
\left(p_\theta(x)\exp\!\,[\ell_{\theta+\delta}(x)-\ell_\theta(x)]\right)^{1-\tau}\,d\mu(x) \\
&=
\int p_\theta(x)^\tau\,
p_\theta(x)^{1-\tau} \exp\!\,[\ell_{\theta+\delta}(x)-\ell_\theta(x)]^{1-\tau}\,d\mu(x) \\
&=
\int p_\theta(x)\,
\exp\!\bigl((1-\tau)\,[\ell_{\theta+\delta}(x)-\ell_\theta(x)]\bigr)\,d\mu(x) \\
&=
\E_\theta\!\left[\exp\!\bigl((1-\tau)\Delta_\ell(X)\bigr)\right],
\end{align*}
where $\Delta_\ell(x):=\ell_{\theta+\delta}(x)-\ell_\theta(x)$.

By a second-order Taylor expansion of $\ell_{\theta+\delta}$ around $\theta$,
\[
\Delta_\ell(x)
=
\delta^\top s_\theta(x) + \frac12\,\delta^\top H_\theta(x)\,\delta
+ O(\|\delta\|^3),
\]
where $s_\theta(x):=\nabla_\theta \ell_\theta(x)$ is the score and
$H_\theta(x):=\nabla_\theta^2 \ell_\theta(x)$ is the Hessian.
Hence
\[
(1-\tau)\Delta_\ell(x)
=
(1-\tau)\,\delta^\top s_\theta(x)
+
\frac{1-\tau}{2}\,\delta^\top H_\theta(x)\,\delta
+
O(\|\delta\|^3).
\]

We now expand $\log A(\delta)=\log \E_\theta[\exp((1-\tau)\Delta_\ell(X))]$
to second order in $\delta$.
Let
\[
U(x):=\delta^\top s_\theta(x),\qquad
V(x):=\frac12\,\delta^\top H_\theta(x)\,\delta.
\]
Then $(1-\tau)\Delta_\ell=(1-\tau)U+(1-\tau)V+O(\|\delta\|^3)$, where $U=O(\|\delta\|)$
and $V=O(\|\delta\|^2)$.
Using the standard cumulant expansion 
\begin{equation*}
\log \E[e^{aU+bV}] = a\,\E[U] + b\,\E[V] + \frac{a^2}{2}\Var(U) + O(\|\delta\|^3),
\end{equation*}
where $\Var(V)=O(\|\delta\|^4)$, as $V=O(\|\delta\|^2)$, is absorbed into $O(\|\delta\|^3)$, we obtain
\begin{equation*} \label{eq:logA-expansion}
\log A(\delta)
=
(1-\tau)\E_\theta[V]
+\frac{(1-\tau)^2}{2}\Var_\theta(U)
+O(\|\delta\|^3).
\end{equation*}
Now, we can compute each term. First, by regularity, $\E_\theta[s_\theta]=0$, hence
$\E_\theta[U]=\delta^\top \E_\theta[s_\theta]=0$ (so there is no linear term).
Next,
\[
\Var_\theta(U)=\E_\theta[U^2]
=\E_\theta[(\delta^\top s_\theta)^2]
=\delta^\top I(\theta)\,\delta.
\]
Finally,
\[
\E_\theta[V]
=\frac12\,\delta^\top \E_\theta[H_\theta(X)]\,\delta
=
-\frac12\,\delta^\top I(\theta)\,\delta,
\]
using $\E_\theta[\partial_{ij}\ell_\theta]= -I_{ij}(\theta)$.

Substituting into \cref{eq:logA-expansion} gives
\[
\log A(\delta)
=
(1-\tau)\left(-\frac12\,\delta^\top I(\theta)\delta\right)
+\frac{(1-\tau)^2}{2}\left(\delta^\top I(\theta)\delta\right)
+O(\|\delta\|^3).
\]
Factor the quadratic form:
\[
\log A(\delta)
=
-\frac{\tau(1-\tau)}{2}\,\delta^\top I(\theta)\,\delta
+O(\|\delta\|^3).
\]

By definition \cref{eq:srfe-param},
\[
D^{\mathrm{SRFE}}_\tau(p_\theta\|p_{\theta+\delta})
=
-\frac{1}{\tau(1-\tau)}\log A(\delta)
=
\frac12\,\delta^\top I(\theta)\,\delta
+O(\|\delta\|^3),
\]
which proves \cref{eq:srfe-local-fisher}. Differentiating twice with respect
to $\theta'$ at $\theta'=\theta$ yields the induced metric
$g^{(\mathrm{SRFE})}_{ij}(\theta)=I_{ij}(\theta)$, proving \cref{eq:srfe-metric-fisher}.
\end{proof}

\hypertarget{SRFE-tailbound}{
\subsection{Proof of Theorem~\ref{thm:srfe-tailbound}} \label{app:srfe-tailbound}
}

\begin{shadedTheorem}[unbreakable]
\srfetailbound*
\end{shadedTheorem}

\begin{proof}
For $\tau\in(0,1)$, Markov's inequality yields
\[
\Pr_q(\Delta\ge a)
=
\Pr_q\!\left(e^{\tau\Delta}\ge e^{\tau a}\right)
\le e^{-\tau a}\,\E_q[e^{\tau\Delta}].
\]
Since $e^{\tau\Delta} = e^{\tau\log p/q} = (p/q)^\tau$, we have
\[
\E_q[e^{\tau\Delta}]
=
\int q(x)\left(\frac{p(x)}{q(x)}\right)^\tau dx
\]
Rearranging SRFE,
\begin{align*}
    \SRFE{P}{Q} &= -\dfrac{1}{\tau(1-\tau)} \log \int q(x)\left(\frac{p(x)}{q(x)}\right)^\tau dx \\
    \exp(-\tau(1 - \tau)) \SRFE{P}{Q}) &= \int q(x)\left(\frac{p(x)}{q(x)}\right)^\tau dx
\end{align*}
Substituting this into the Markov bound gives
\begin{align*}
    \Pr_q(\Delta\ge a)
    &\le
    \exp(-\tau a)\,\exp(-\tau(1 - \tau)) \SRFE{P}{Q}) \\
    &=
    \exp(-\tau a -\tau(1 - \tau)) \SRFE{P}{Q}),
\end{align*}
which matches \cref{eq:srfe-tailbound}.
\end{proof}

\hypertarget{SRFE-gibbs}{
\subsection{Proof of Theorem~\ref{thm:srfe-gibbs-var}} \label{app:srfe-gibbs-var}
}

\begin{shadedTheorem}
\srfegibbsvar*
\end{shadedTheorem}

\begin{proof}
Consider the functional
$J(r) \coloneq \frac{1}{\tau}\KL{r}{p} + \frac{1}{1-\tau}\KL{r}{q}$
over densities $r$ with $\int r=1$.
Writing out the KL divergences and collecting terms gives
\begin{align*}
    J(r)
    &=
    \int r\log r
    -\int r\Big(\tau\log p + (1-\tau)\log q\Big)
    +\text{const}(\tau,p,q).
\end{align*}
By the Gibbs (log-sum) variational principle,
the minimizer satisfies
$
\log r = \tau\log p + (1-\tau)\log q - \log Z + \text{const},
$
hence $r=r_\tau$.
Substituting $r_\tau$ back into $J(r)$ yields
$
J(r_\tau)= -\frac{1}{\tau(1-\tau)}\log\int p^\tau q^{1-\tau}
= \SRFE{P}{Q}.
$
\end{proof}

\hypertarget{SRFE-KL-bounds}{
\subsection{Proof of Corollary~\ref{cor:srfe-kl-upper}} \label{app:srfe-kl-upper}
}

\begin{shadedTheorem}[unbreakable]
\srfeklupper*
\end{shadedTheorem}

\begin{proof}
Using the variational characterization of SRFE, define
\[
J(r):=\frac{1}{\tau}\KL{r}{P}+\frac{1}{1-\tau}\KL{r}{Q},
\qquad r\in\mathcal P(\mathcal X).
\]
Then $\SRFE{P}{Q}=\min_r J(r)$, so for any admissible $r$ we have
$\SRFE{P}{Q}\le J(r)$.

Choosing $r=P$ gives
\[
J(P)=\frac{1}{\tau}\KL{P}{P}+\frac{1}{1-\tau}\KL{P}{Q}
=\frac{1}{1-\tau}\KL{P}{Q},
\]
and choosing $r=Q$ gives
\[
J(Q)=\frac{1}{\tau}\KL{Q}{P}+\frac{1}{1-\tau}\KL{Q}{Q}
=\frac{1}{\tau}\KL{Q}{P}.
\]
Taking the minimum of these two upper bounds yields the stated result.
\end{proof}

\paragraph{Why SRFE is upper bounded by scaled KL endpoints}
\label{rem:srfe-kl-upper-intuition}
The variational form
\[
\SRFE{P}{Q}
=
\min_{r\in\mathcal P(\mathcal X)}
\left\{
\frac{1}{\tau}\KL{r}{P}+\frac{1}{1-\tau}\KL{r}{Q}
\right\}
\]
shows that SRFE is the \emph{best} achievable weighted compromise between proximity to $P$ and proximity to $Q$ in KL geometry. Evaluating this objective at the two extreme choices $r=P$ and $r=Q$ yields feasible (not necessarily optimal) values
\[
J(P)=\frac{1}{1-\tau}\KL{P}{Q},
\qquad
J(Q)=\frac{1}{\tau}\KL{Q}{P},
\]
and since a minimum is never larger than any feasible value, we obtain the endpoint bounds
\[
\SRFE{P}{Q}
\le \min\!\left\{\frac{1}{1-\tau}\KL{P}{Q},\,\frac{1}{\tau}\KL{Q}{P}\right\}.
\]
Thus, SRFE cannot exceed the scaled forward or reverse KL costs of moving directly to either endpoint; instead it selects an intermediate escort distribution that improves upon both whenever possible.

\hypertarget{SRFE-MDL}{
\subsection{Proof of Corollary~\ref{cor:mdl-srfe-tail}} \label{app:mdl-srfe-tail}
}

\begin{shadedTheorem}
\mdlsrfetail*
\end{shadedTheorem}

\begin{proof}
By definition, $\Excess(X)=\log\frac{p(X)}{q(X)}$, so
\[
\Pr_{X\sim Q}\!\big[\Excess(X)\ge a\big]
=
\Pr_{X\sim Q}\!\left[\log\frac{p(X)}{q(X)}\ge a\right].
\]
Applying \cref{thm:srfe-tailbound} with $\Delta(X)=\log\frac{p(X)}{q(X)}$ yields \cref{eq:mdl-excess-tail}.
\end{proof}
\section{Synthetic Experiments} \label{app:experiments-synthetic}

We provide additional material regarding the synthetic experiments outlined in the \cref{sec:experiments}. We only presented our results on the 2D Gaussian mixture model and multi-class classification in the main text, but we provide additional analyses and experiments on gradient dynamics \pcref{app:grad-dynamics} and label smoothing \pcref{app:label-smoothing} here.

\subsection{2D Gaussian Mixture Model}

The target distribution is a three-component Gaussian mixture with means $\mu_1=[-3, -1]$, $\mu_2=[-4, 2]$, and $\mu_3=[4, y_3]$ where $y_3\in\{4, 0, -3\}$ with shared isotropic diagonal covariances $\Sigma_1 = \Sigma_2 = \Sigma_3 = 0.5 I_2$. The mixture weights were $w=[0.3, 0.3, 0.4]$.

The variational posterior was set to a Gaussian with diagonal covariance (two trainable parameters $\{\mu,\sigma\}$). The experiment was repeated for 10 random seeds to evaluate the sensitivity of the objective to initializations. Note that the initializations for each combination of the target distribution and training objective were kept equal to eliminate any differences due to randomness in initializations. We evaluated the forward and reverse KL divergences and SRFE with $\tau=\{0.3, 0.5, 0.7\}$. The model was trained with 1500 samples for 100 epochs using the Adam optimizer with a learning rate of 0.05.

Once trained, we evaluated the model on a 200-by-200 grid on $x\in[-7, 7]$ and $y\in[-6, 9]$. We plotted the contours in addition to the converged ellipses at $2\sigma$ for each initialization. The dotted gray line connects the initial and converged mean for each of the initializations.

\subsection{Multi-Class Classification}

A small MLP (layers=$[n_\text{dim}, 16, 16, n_\text{classes}]$) was trained on a synthetically generated 5-class classification dataset using SRFE with $\tau=[0.1, 0.3, 0.5, 0.7, 0.9]$. The dataset was generated using \texttt{scikit-learn}'s \citep{scikit-learn} \texttt{make\_classification} with \texttt{n\_samples=2500}, \texttt{n\_features=12}, \texttt{n\_informative=8}, and \texttt{n\_redundant=2} in an 80-20 train-test split. Experiments were repeated with 15 random initializations per value of $\tau$. Models were trained using the Adam optimizer with a learning rate of 0.01 for 50 epochs.

Models were evaluated on the test split for top-1 classification accuracy, mean prediction entropy, and mean prediction confidence (probability assigned to predicted class). We plot the mean $\pm$ standard deviation for each objective. We additionally plot the kernel density estimate (KDE) of each objectives per-sample prediction entropy.

\subsection{Gradient Dynamics} \label{app:grad-dynamics}

We compare the gradient of the SRFE objective to the standard cross-entropy (CE) objective. We provide a mathematical derivation of the gradients and demonstrate their difference in a synthetic multi-class classification setting.

\subsubsection{Theoretical comparison} \label{app:grad-dynamics-theory}

Consider a $C$-class classification problem with target probability distribution $p\in[0,1]^C$. Let $z\in\R^C$ denote a model's logits and $q\in\R^C$ be the predicted probability distribution for a given sample obtained by calculating the softmax of $z$:
\begin{equation}
    q_k = \dfrac{\exp(z_k)}{\sum_j \exp(z_j)}.
\end{equation}

\paragraph{SRFE without label smoothing reduces to cross-entropy} \label{sec:srfe-reduces-to-cross-entropy}
If we consider $C$-class classification with one-hot labels (without label smoothing), the SRFE objective simplifies to
\begin{align*}
    \gL_{\text{SRFE}}
    &= -\dfrac1N \cdot \dfrac{1}{\tau(1-\tau)} \sum_{i=1}^N \log \sum_{c=1}^C p_{i,c}^\tau q_{i,c}^{1-\tau} \\
    &= -\dfrac{1}{N\tau(1-\tau)} \sum_{i=1}^N \log p_{i,c^*}^\tau q_{i,c^*}^{1-\tau} \\
    &= -\dfrac{1}{N\tau(1-\tau)} \sum_{i=1}^N \log q_{i,c^*}^{1-\tau} \\
    &= -\dfrac{1}{N\tau(1-\tau)} \sum_{i=1}^N (1-\tau) \log q_{i,c^*} \\
    &= \dfrac1\tau \left(-\dfrac1N \sum_{i=1}^N \log q_{i,c^*}\right),
\end{align*}
where $p_{i,c^*} = p_{i,c}(c=y_i) = 1$, $p_{i,c}(c\neq y_i) = 0$, and $q_{i,c^*} = q_{i,c}(c=y_i)$.
Thus for one-hot labels, SRFE reduces to the scaled cross-entropy loss. Therefore, to fully leverage the properties afforded by SRFE, soft labels are required (via practices such as label smoothing) in the discrete setting.

\paragraph{CE gradient}

We begin with the cross-entropy objective
\begin{equation}
    \mathcal{L}_{\mathrm{CE}} = -\sum_c p_c \log q_c.
\end{equation}
The derivative of $\mathcal{L}_{\mathrm{CE}}$ is 
\begin{equation}
    \dfrac{\partial \mathcal{L}_{\mathrm{CE}}}{\partial z_k}
    = -\dfrac{\partial}{\partial z_k} \sum_c p_c \log q_c
    = -\sum_c p_c \dfrac{\partial \log q_c}{\partial z_k}
    = -\sum_c p_c \dfrac{1}{q_c} \dfrac{\partial q_c}{\partial z_k}.
    \label{eq:dCE-dz}
\end{equation}
Substituting the derivative of the softmax function,
\begin{equation}
    \dfrac{\partial q_c}{\partial z_k}
    = q_c (\delta_{ck} - q_k),
    \label{eq:softmax-derivative}
\end{equation}
where $\delta_{ck}$ is the Kronecker delta ($\delta_{ck}=1$ if $c=k$ and 0 otherwise), into \cref{eq:dCE-dz}, we get
\begin{align}
    \dfrac{\partial \mathcal{L}_{\mathrm{CE}}}{\partial z_k}
    &= -\sum_c p_c \dfrac{1}{q_c} q_c (\delta_{ck} - q_k) \nonumber \\
    &= -\sum_c p_c (\delta_{ck} - q_k) \nonumber \\
    &= -\sum_c p_c \delta_{ck} + \sum_c p_c q_k \nonumber \\
    &= q_k -p_k.
\end{align}
Across all $k$ classes, we have the gradient
\begin{equation}
    \nabla_{z} \mathcal{L}_{\mathrm{CE}} = q - p.
\end{equation}

\paragraph{SRFE gradient}

The discrete SRFE objective is defined as
\begin{equation}
    \mathcal{L}_{\mathrm{SRFE}} = -\dfrac{1}{\tau(1-\tau)} \log F(\tau),\quad\text{where}\quad F(\tau) = \sum_c p_c^\tau q_c^{1-\tau},
\end{equation}
for $\tau\in(0,1)$. The derivative of $\mathcal{L}_{\mathrm{SRFE}}$ with respect to $z_k$ is
\begin{equation}
    \dfrac{\partial \mathcal{L}_{\mathrm{SRFE}}}{\partial z_k}
    = -\dfrac{1}{\tau(1-\tau)} \dfrac{\partial}{\partial z_k} \log F(\tau)
    = -\dfrac{1}{\tau(1-\tau)} \dfrac{1}{F(\tau)} \dfrac{\partial F(\tau)}{\partial z_k}.
    \label{eq:dSRFE-dz}
\end{equation}
We evaluate the last partial derivative
\begin{equation}
    \dfrac{\partial F(\tau)}{\partial z_k}
    = \dfrac{\partial}{\partial z_k} \sum_c p_c^\tau q_c^{1-\tau}
    = \sum_c p_c^\tau \dfrac{\partial q_c^{1-\tau}}{\partial z_k} 
    = \sum_c p_c^\tau (1-\tau) q_c^{-\tau} \dfrac{\partial q_c}{\partial z_k}.
    \label{eq:dFtau-dz}
\end{equation}
Substituting the softmax derivative from \cref{eq:softmax-derivative} into \cref{eq:dFtau-dz}, we get
\begin{align}
    \dfrac{\partial F(\tau)}{\partial z_k} 
    &= \sum_c p_c^\tau (1-\tau) q_c^{-\tau} q_c (\delta_{ck} - q_k) \nonumber \\
    &= (1-\tau) \sum_c p_c^\tau q_c^{1-\tau} (\delta_{ck} - q_k) \nonumber \\
    &= (1-\tau) \left[ \sum_c p_c^\tau q_c^{1-\tau} \delta_{ck} - \sum_c p_c^\tau q_c^{1-\tau} q_k \right] \nonumber \\
    &= (1-\tau) \left[ p_k^\tau q_k^{1-\tau} - \sum_c p_c^\tau q_c^{1-\tau} q_k \right] \nonumber \\
    &= (1-\tau) \left[ p_k^\tau q_k^{1-\tau} - q_k F(\tau) \right].
\end{align}
Substituting into \cref{eq:dSRFE-dz},
\begin{align}
    \dfrac{\partial \mathcal{L}_{\mathrm{SRFE}}}{\partial z_k} 
    &= -\dfrac{1}{\tau(1-\tau)} \dfrac{1}{F(\tau)} (1-\tau) \left[ p_k^\tau q_k^{1-\tau} - q_k F(\tau) \right] \nonumber \\
    &= -\dfrac{1}{\tau F(\tau)} \left[ p_k^\tau q_k^{1-\tau} - q_k F(\tau) \right] \nonumber \\
    &= \dfrac1\tau \left[q_k - \dfrac{p_k^\tau q_k^{1-\tau}}{F(\tau)} \right] \nonumber \\
    &= \dfrac1\tau (q_k - r_\tau).
\end{align}
where $r_\tau=p_k^\tau q_k^{1-\tau} / F(\tau)$ is the escort distribution introduced in \cref{lem:srfe-gradient}.
Across all $k$ classes, we have
\begin{equation}
    \nabla_z \mathcal{L}_{\mathrm{SRFE}} = \dfrac1\tau (q - r_\tau).
\end{equation}
Consequently, SRFE does \emph{not} optimize directly against the smoothed target $p$; instead, it optimizes against the transformed distribution $r_\tau$.

\subsubsection{Numerical illustration}

Consider a classification task with $C=10$ classes and a label smoothing factor of $\alpha=0.1$. We assume the model incorrectly predicts class $c=\hat{y}$ (where $\hat{y} \neq y$). 

\begin{itemize}
    \item \textbf{Target distribution ($p$):} The smoothed label for the correct class is $p_y = 1 - \alpha + \alpha / C = 0.91$. The labels for the remaining classes $i\neq y$ are $p_i = \alpha / C = 0.1 / 10 = 0.01$.
    
    \item \textbf{Predicted distribution ($q$):} For $\hat{y}$ to be the predicted class, we only consider $q_{\hat{y}} \ge 1 / C$ with the remaining probability assigned evenly over the other classes. We define the model's confidence in this incorrect prediction as $q_{\hat{y}} = \theta$ where $\theta \in [1/C, 1)$.
    
    \item \textbf{Gradients:} For each level of confidence, we compute the gradient of the loss with respect to the predicted (incorrect) class logit. We repeat the same experiment for SRFE with $\tau\in\{0.1, 0.3, 0.5, 0.7, 0.9\}$. We also compare the gradients of CE and SRFE at $\tau=\{0.3, 0.5, 0.7\}$ for $C=\{2, 5, 10, 50, 100\}$ to illustrate how the gradient change based on the number of classes.
\end{itemize}

\cref{fig:gradient-confidence_tau-sweep} shows the gradient as a function of model confidence for CE and SRFE. The CE gradient is clearly linear in the model confidence and produces a larger gradient as the model's confidence $q_{\text{incorrect}}$ increases. In contrast, the SRFE gradients display a non-monotonic profile -- for all values of $\tau$ the SRFE gradients first increase then rapidly decrease to zero as the model confidence approaches 1. We additionally observe that the SRFE gradient approaches that of CE as $\tau\to 1$, which aligns with our analysis that SRFE approaches CE as $\tau\to 1$.

\begin{figure}[htbp]
    \centering
    \includegraphics[width=0.5\linewidth]{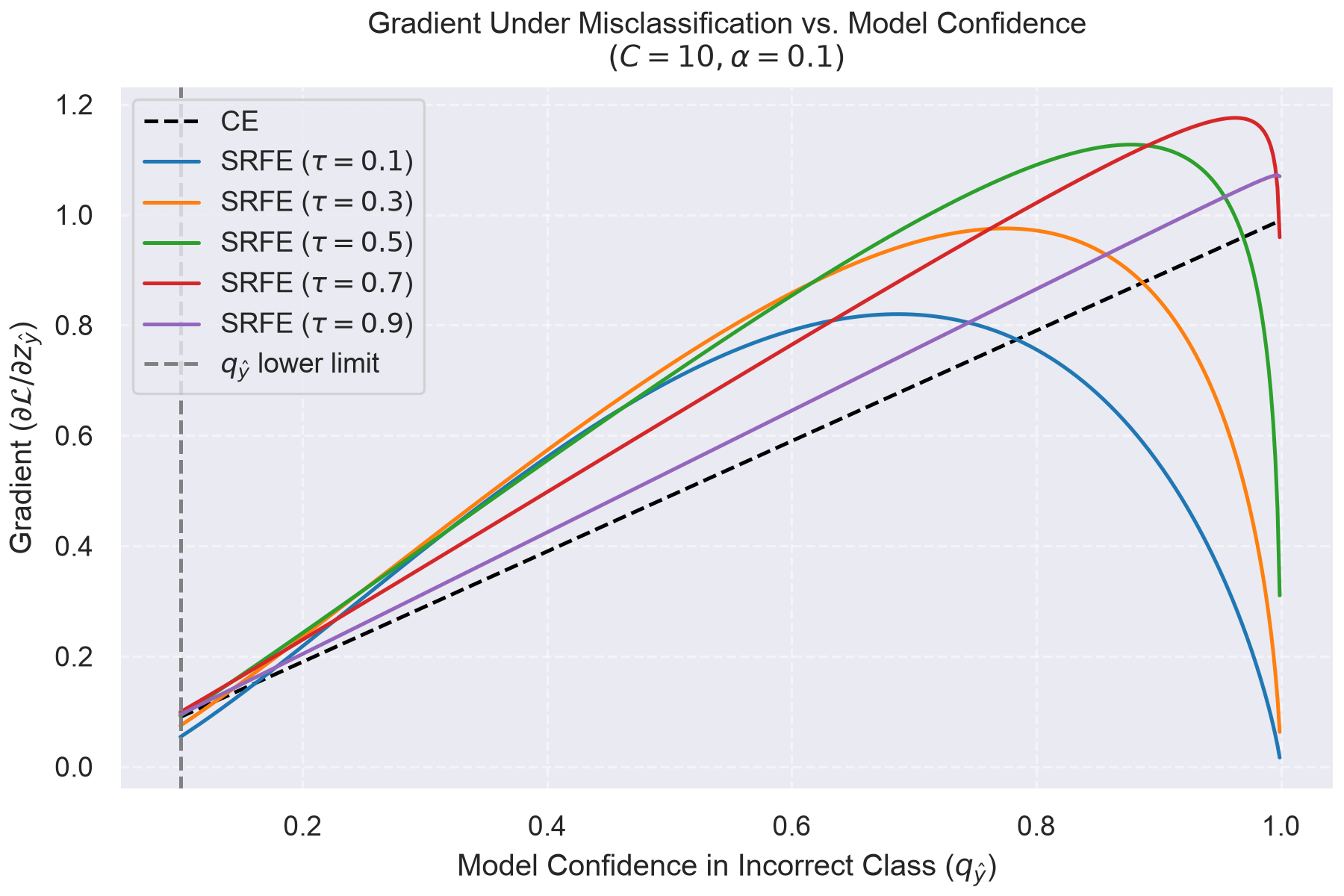}
    \caption{Gradients for CE and SRFE with $\tau=\{0.1, 0.3, 0.5, 0.7, 0.9\}$ as a function of model confidence in misclassification for 10-class classification.}
    \label{fig:gradient-confidence_tau-sweep}
\end{figure}

\cref{fig:gradient-confidence_class-sweep} compares the gradient profile for $C$-class classification with $C=\{2, 5, 10, 50, 100\}$ for SRFE with $\tau=\{0.3, 0.5, 0.7\}$ and CE. For a given value of $\tau$, the SRFE objective increases the overall magnitude of the gradient as the number of classes increases. This does not occur with CE as the gradient does not vary with the number of classes.

\begin{figure}[htbp]
    \centering
    \includegraphics[width=0.8\linewidth]{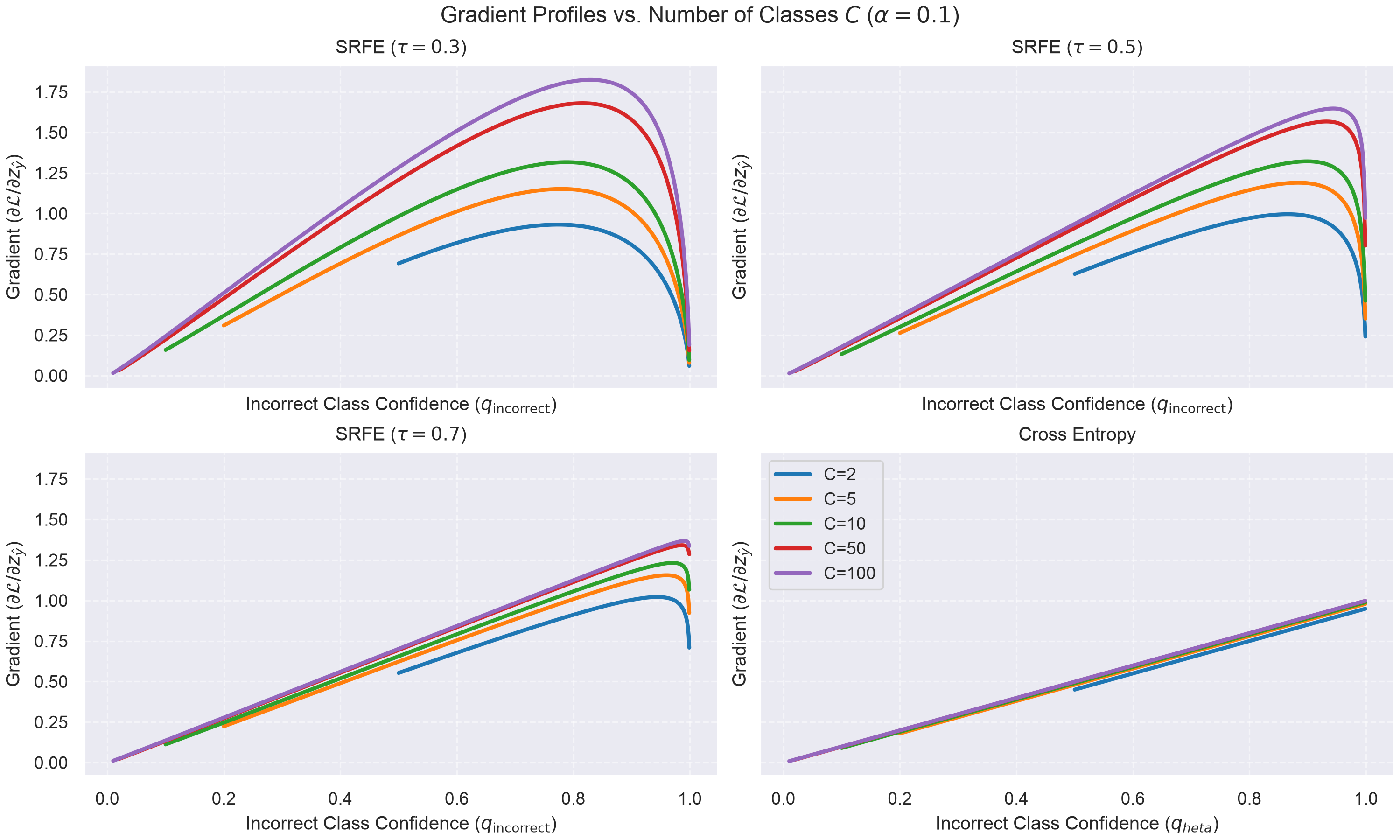}
    \caption{Gradients for SRFE with $\tau=0.3, 0.5, 0.7$ and CE as a function of model confidence in misclassification for $C$-class classification.}
    \label{fig:gradient-confidence_class-sweep}
\end{figure}

\subsection{Label smoothing} \label{app:label-smoothing}

In the following section, we theoretically show that label smoothing values that are standard for CE ($\alpha=0.1$) can be overly aggressive and detrimental to SRFE.

\subsubsection{Theoretical motivation} \label{app:label-smoothing-theory}

\paragraph{Comparing CE and SRFE gradients}

Consider the scenario where the correct class is $c=y$ with the smoothed labels $p_y=1-\alpha + \alpha / C$ and $p_c=\alpha / C$ for $c\neq y$, where $\alpha=0.1$ (commonly used value in classification with CE objective). Assume the model to be near initialization such that the predicted probabilities $q_i = 1 / C$ for $i=\{1,\dots,C\}$ (random guess).

For $\tau\in(0,1)$, exponentiation by $\tau$ amplifies the probabilities assigned to incorrect classes in the target distribution as $0\le p_c\le 1$:
\begin{equation}
    r_{\tau,c} = \dfrac{p_c^\tau q_c^{1-\tau}}{F(\tau)}.
\end{equation}
Consider when $C=10$ such that $q_i=0.1$ for $i=\{1,\dots,10\}$. Then, the target (smoothed labels) for CE are 
\begin{equation}
    p_y = 0.91,\quad p_c=0.01\quad{\text{for}}\, c\neq y.
\end{equation}
For SRFE with $\tau=0.5$,
\begin{equation}
    F(\tau) = \sum_{i=1}^C p_i^\tau q_i^{1-\tau} = 0.91^{0.5} \cdot 0.1^{0.5} + 9(0.01^{0.5} \cdot 0.1^{0.5}) \approx 0.58627,
\end{equation}
and the escort target probabilities are given by 
\begin{align}
    r_{\tau,y} &= \dfrac{p_y^\tau q_y^{1-\tau}}{F(\tau)} = \dfrac{0.91^{0.5} \cdot 0.1^{0.5}}{0.58627} \approx 0.5145,  \\
    r_{\tau,c} &= \dfrac{p_c^\tau q_c^{1-\tau}}{F(\tau)} = \dfrac{0.01^{0.5} \cdot 0.1^{0.5}}{0.58627} \approx 0.05394 \quad{\text{for}}\, c\neq y.
\end{align}
Thus, the target for the correct class is significantly smaller compared to CE. This motivates the need for $\tau$-specific values of label smoothing.

\paragraph{Asymptotic behavior}

We extend the analysis above to the asymptotic case as the number of classes $C$ grows.
For a given $C$, at initialization ($q_i = 1/C \,\forall i\in \{1,\dots,C\}$)
\begin{equation}
    r_{\tau,c}
    = \dfrac{p_c^\tau q_c^{1-\tau}}{\sum_{i=1}^C p_i^\tau q_i^{1-\tau}}
    = \dfrac{p_c^\tau (1/C)^{1-\tau}}{\sum_{i=1}^C p_i^\tau (1/C)^{1-\tau}}
    = \dfrac{p_c^\tau}{\sum_{i=1}^C p_i^\tau}
    = \dfrac{p_c^\tau}{p_c^\tau + \sum_{i\neq c} p_i^\tau},
    \label{eq:escort-target-uniform-q}
\end{equation}
for each class $c\in \{1,\dots,C\}$.
Using label smoothing $\alpha\in(0,1)$, we can substitute $p_y=1-\alpha+\alpha/C$ and $p_c=\alpha/C$ for $c\neq y$ to get
\begin{equation}
    r_{\tau,y}
    = \dfrac{(1-\alpha+\alpha/C)^\tau}{(1-\alpha+\alpha/C)^\tau + (C-1) (\alpha/C)^\tau}.
\end{equation}
As $C\to\infty$, the second term in the denominator simplifies to
\begin{equation}
    \lim_{C\to\infty} (C-1) (\alpha/C)^\tau
    = \lim_{C\to\infty} \alpha^\tau \left(\dfrac{C-1}{C^\tau}\right)
    = \lim_{C\to\infty} \alpha^\tau \left(C^{1-\tau} - C^{-\tau}\right).
\end{equation}
Since $\tau\in(0,1)\to 1-\tau>0$, we have
\begin{equation}
    \lim_{C\to\infty} C^{1-\tau} = \infty \qquad \lim_{C\to\infty} C^{-\tau} = 0.
\end{equation}
Substituting these results, we get
\begin{equation}
    \lim_{C\to\infty} (C-1) (\alpha/C)^\tau
    = \lim_{C\to\infty} \alpha^\tau \cdot \infty
    = \infty,
\end{equation}
and, finally,
\begin{equation}
    \lim_{C\to\infty} r_{\tau,y}
    = \dfrac{(1-\alpha)^\tau}{(1-\alpha)^\tau + \infty}
    = 0.
\end{equation}

\paragraph{Label smoothing selection}

As shown in \cref{app:grad-dynamics-theory}, SRFE reduces to a scaled CE without label smoothing. On the other hand, na\"ively using label smoothing values for CE with SRFE leads to overly diffuse target labels that do not provide strong learning signals in the classification setting. For this, we attempt to derive a scaling law for selecting the SRFE label smoothing factor as a function of $\tau\in(0,1)$ and number of classes $C$ such that the escort target probability for the true class receives a specified proportion, $1-\alpha_{\text{eff}}$, of the probability mass for an effective label smoothing parameter $\alpha_{\text{eff}}\in(0,1)$.

\begin{proposition}[SRFE label smoothing selection] \label{prop:srfe-label-smoothing-exact}
Let $\alpha_{\text{eff}}\in(0,1)$ denote the effective label smoothing for SRFE with $\tau$ such that the correct class escort target label $r_{\tau,y} = 1 - \alpha_{\text{eff}}$ and $r_{\tau,c} = \alpha_{\text{eff}} / C$ for $c\neq y$. Then, the label smoothing $\alpha$ for a desired value of $\alpha_{\text{eff}}$ should be selected as \begin{equation}
    \alpha(C,\tau,\alpha_{\text{eff}})
    = \dfrac{C}{\left(\dfrac{1-\alpha_{\text{eff}}}{\alpha_{\text{eff}}}\right)^{1/\tau} \cdot (C-1)^{1/\tau} + (C - 1)}.
    \label{eq:srfe-label-smoothing-exact}
\end{equation}
\end{proposition}

\begin{proof}
Consider a classification task with $C$ classes and standard label smoothing with $\alpha\in(0,1)$ such that $p_y = 1 - \alpha + \alpha/C$ and $p_c = \alpha / C$ for $c\neq y$. Assume the model being trained to be near initialization so that $q_i=1/C\,\forall i\in\{1,\dots,C\}$. The escort target distribution for the true class is given by \cref{eq:escort-target-uniform-q}:
\begin{equation}
    r_{\tau,y}
    = \dfrac{p_y^\tau}{p_y^\tau + \sum_{i\neq y} p_i^\tau}
    = 1 - \alpha_{\text{eff}},
\end{equation}
where we solve for $\alpha$ in the second equality. With some basic algebra, we get
\begin{align}
    p_y^\tau
    &= (1-\alpha_{\text{eff}}) p_y^\tau + (1-\alpha_{\text{eff}}) \sum_{i\neq y} p_i^\tau \\
    \alpha_{\text{eff}} \cdot p_y^\tau
    &= (1-\alpha_{\text{eff}}) (C-1) p_c^\tau \\
    \left(\dfrac{p_y}{p_c}\right)^\tau
    &= \dfrac{1-\alpha_{\text{eff}}}{\alpha_{\text{eff}}} \cdot (C-1) \\
    \dfrac{p_y}{p_c}
    &= \left(\dfrac{1-\alpha_{\text{eff}}}{\alpha_{\text{eff}}}\right)^{1/\tau} \cdot (C-1)^{1/\tau}.
\end{align}
Substituting in for $p_y$ and $p_c$ and simplifying,
\begin{align}
    \dfrac{1-\alpha+\alpha / C}{\alpha / C}
    &= \left(\dfrac{1-\alpha_{\text{eff}}}{\alpha_{\text{eff}}}\right)^{1/\tau} \cdot (C-1)^{1/\tau} \\
    \dfrac{C-C\alpha+\alpha}{\alpha}
    &= \left(\dfrac{1-\alpha_{\text{eff}}}{\alpha_{\text{eff}}}\right)^{1/\tau} \cdot (C-1)^{1/\tau} \\
    \dfrac{C}{\alpha} - (C - 1)
    &= \left(\dfrac{1-\alpha_{\text{eff}}}{\alpha_{\text{eff}}}\right)^{1/\tau} \cdot (C-1)^{1/\tau} \\
    \dfrac{C}{\alpha}
    &= \left(\dfrac{1-\alpha_{\text{eff}}}{\alpha_{\text{eff}}}\right)^{1/\tau} \cdot (C-1)^{1/\tau} + (C - 1) \\
    \alpha
    &= \dfrac{C}{\left(\dfrac{1-\alpha_{\text{eff}}}{\alpha_{\text{eff}}}\right)^{1/\tau} \cdot (C-1)^{1/\tau} + (C - 1)}
\end{align}
\end{proof}



\subsubsection{Synthetic experiment}

We validate the label smoothing result presented in \cref{prop:srfe-label-smoothing-exact} with a synthetic demonstration comparing the effect of a fixed label smoothing value ($\alpha=0.1$) and a dynamically selected label smoothing value $\alpha_{\text{eff}}$ based on the number of classes.

\cref{fig:label-smoothing} shows that the effective correct class target probability $r_{\tau,y}$ decays as the number of classes increases when the typical label smoothing $\alpha=0.1$ is used with SRFE. Even with a relatively small number of classes ($C=10$), the probability allocated to the correct class label decreases significantly ($r_{\tau,y}=\{0.301, 0.515, 0.723\}$ for $\tau=\{0.3, 0.5, 0.7\}$, respectively). At $C=1000$, the decay is even more prominent as $r_{\tau,y}=\{0.0151, 0.0867, 0.3698\}$ for the same values of $\tau$. The horizontal line at the top of the figure shows that calibrating $\alpha$ based on the number of classes ensures that the target label retains a desired level of probability (not pictured for $\tau=\{0.3, 0.5\}$ as it overlaps with $\tau=0.7$).

\begin{figure}[htbp]
    \centering
    \includegraphics[width=0.7\linewidth]{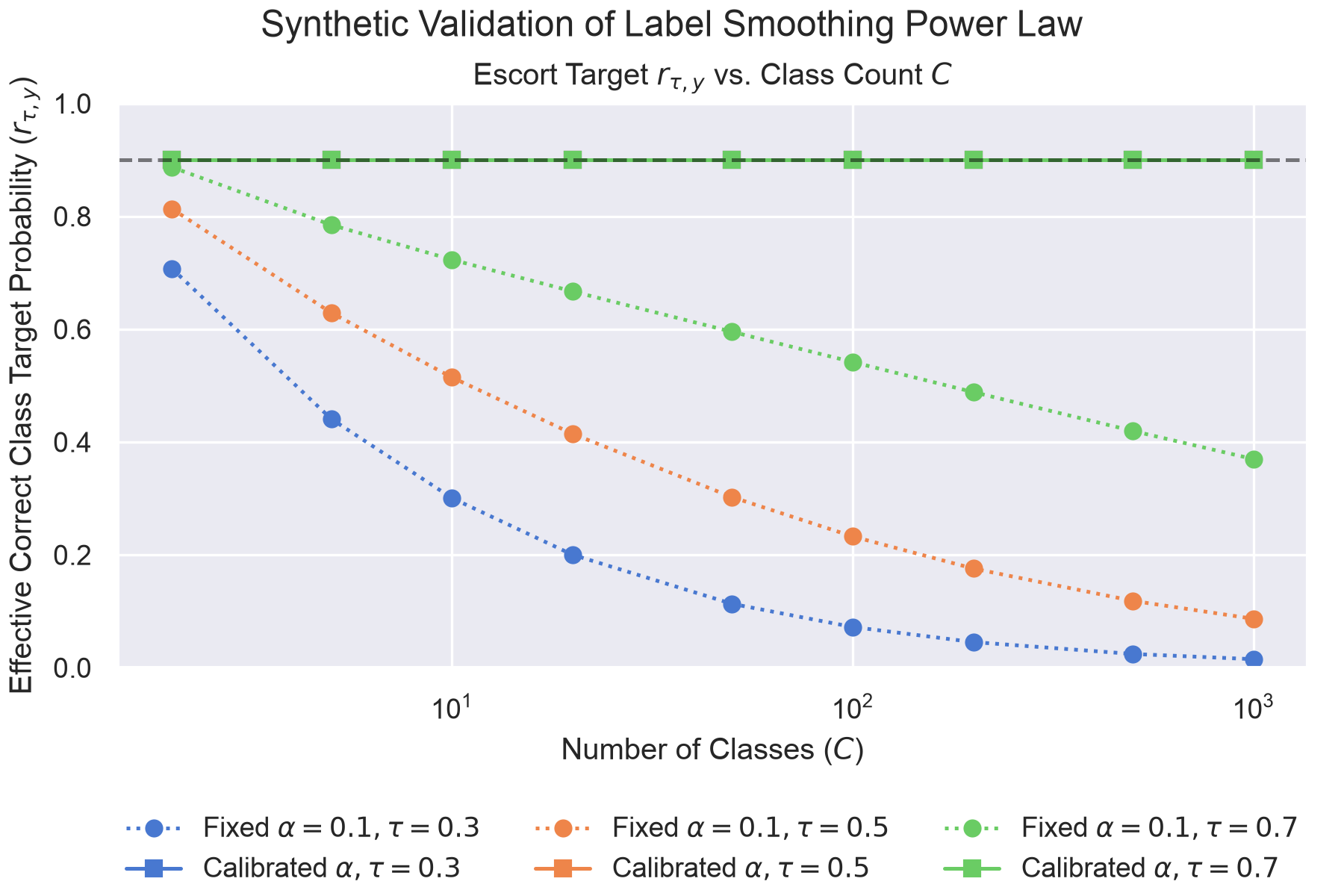}
    \caption{Correct class probability after fixed (dashed line) and calibrated (solid line) label smoothing for SRFE at $\tau=\{0.3, 0.5, 0.7\}$ with number of classes $C=\{2, 5, 10, 20, 50, 100, 200, 500, 1000\}$.}
    \label{fig:label-smoothing}
\end{figure}

\section{Long-Tailed CIFAR Experiments} \label{app:experiments-cifar}

We evaluated the performance of SRFE on the long-tailed variants of the CIFAR-10/100 datasets (CIFAR-10/100-LT) \citep{Cui2019}. We compare the performance of models trained using the SRFE objective against a standard cross-entropy (CE) baseline.

\paragraph{Dataset}

We constructed the CIFAR-10/100-LT datasets following the settings described in \cite{Cui2019}.
Briefly, for each of the classes $c\in\{0,1,\dots,C-1\}$ in the original dataset, a subset of training images for class $c$, $N_c=\text{int}(N_0 \cdot \beta^{c/(C-1)})$ is sampled, where $N_0$ is the maximum class size (5000 for CIFAR-10 and 500 for CIFAR-100) and $\beta=1/\textrm{IF} = N_{C-1} / N_0$ is the decay ratio (the imbalance factor (IF) is more commonly used to describe the imbalance of the dataset). In our experiments, we train and test over $\textrm{IF}=\{10,50,100\}$ for both CIFAR-10-LT and CIFAR-100-LT as is commonly done in long-tailed image classification literature \citep{tang2020longtailed}.

\paragraph{Architecture \& Training}

We modified a standard ResNet-32 \citep{He2016} backbone for $32\times32$ images. The output of the final average pooling layer was mapped to a $C$-dimensional fully-connected classifier. Each model was trained on a single NVIDIA A6000 GPU on a server equipped with four GPUs.

\subsection{Setup} \label{app:cifar-setup}

We performed a hyperparameter optimization study to determine the optimal settings for CIFAR-10-LT and CIFAR-100-LT at the three levels of class imbalance. We performed the same hyperparameter optimization for CE to ensure a fairer comparison to SRFE (instead of relying on a single reported result).

\paragraph{Search Space} \label{subsec:search_space}

We optimized hyperparameters such as the optimizer, learning rate, weight decay, label smoothing, batch size, learning rate scheduler, and $\tau$ when using the SRFE objective. The search space consisted of both continuous distributions (sampled via log-uniform or uniform distributions) and discrete categorical choices. The search space is summarized in \cref{tab:search_space}.

\begin{table}[htbp]
\centering
\caption{Hyperparameter optimization search space parameters.}
\label{tab:search_space}
\begin{tabular}{lll}
\toprule
\textbf{Hyperparameter} & \textbf{Type} & \textbf{Range [] or Choices \{\}} \\
\midrule
Optimizer       & Categorical & \{SGD, AdamW\} \\
Learning Rate   & Continuous  & SGD: $[10^{-3}, 0.5]$ (Log-Uniform) \\
& & AdamW: $[10^{-4}, 0.5]$ (Log-Uniform) \\
Weight Decay    & Continuous  & $[10^{-5}, 10^{-2}]$ (Log-Uniform) \\
Label Smoothing & Continuous  & $[0.0, 0.2]$ (Uniform) \\
Batch Size      & Categorical & $\{32, 64, 128, 256, 512\}$ \\
Scheduler       & Categorical & \{Multistep w/ Warmup, Cosine, One Cycle\} \\
SRFE $\tau$     & Continuous  & $[0.01, 0.99]$ (Uniform) \\
\bottomrule
\end{tabular}
\end{table}

\paragraph{Optimization strategy}

We trained a ResNet-32 based architecture using the SRFE or CE objective for 200 epochs with a validation split of 10\% (sampled from the training set). Note that the validation split is balanced unlike the training set. Each model was trained on a single NVIDIA A6000 GPU and took between 60 to 1000 seconds depending on the batch size. We used Optuna\footnote{\url{https://optuna.org/}} \citep{optuna}, an open-source hyperparameter optimization framework, for our optimization studies. 

We ran a separate study for each combination of dataset and IF to determine the optimal set of hyperparameters. We used the built-in \texttt{AutoSampler}\footnote{\url{https://hub.optuna.org/samplers/auto_sampler/}} \citep{optunahub} with a budget of 600 trials for SRFE and 150 trials for CE (completed or pruned) per study. We increased the number of SRFE trials to account for the additional parameter $\tau$. We pruned trials using the \texttt{MedianPruner}\footnote{\url{https://optuna.readthedocs.io/en/stable/reference/generated/optuna.pruners.MedianPruner.html}} with 40 startup trials and 50 warmup epochs per trial. Models were evaluated on the validation set at the end of each epoch and the best performing checkpoint was saved per trial.

\paragraph{Evaluation}

We evaluated the best checkpoint (as measured by the validation accuracy) from each completed trial on the balanced test set (the standard test split for CIFAR-10/100). We report overall, few-shot, medium-shot, and many-shot accuracy as well as expected calibration error (ECE) for each study using only the top 10\% of best performing models (as measured by validation accuracy) on the balanced test set.

\subsection{Results} \label{app:cifar-results}

\paragraph{Test set performance}

We summarize the test set performance of the top 10\% of models per study in \cref{tab:cifar_10_lt_top10_summary,tab:cifar_100_lt_top10_summary}. Note that CIFAR-10-LT with IF=10 and 50, and CIFAR-100-LT with IF=10 are blank as there are no classes that fit into the few-shot category. Note that the trials column represents the number of trials that were completed from which the top 10\% of models (in parentheses) were used to calculate statistics.

\begin{table}[htbp]
\centering
\caption{Hyperparameter optimization evaluation summary on CIFAR-10-LT. Results report the performance of the top-$k$ subset of validation-selected models from completed trials across imbalance factors.}
\label{tab:cifar_10_lt_top10_summary}
\begin{tabular}{llcccccc}
\toprule
\multirow{2}{*}{\textbf{IF}} & \multirow{2}{*}{\textbf{Loss}} & \multirow{2}{*}{\textbf{Trials ($k$)}} & \multicolumn{4}{c}{\textbf{Accuracy (\%)}} & \multirow{2}{*}{\textbf{ECE (\%)}} \\
\cmidrule(lr){4-7}
 & & & \textbf{Overall} & \textbf{Many-shot} & \textbf{Med-shot} & \textbf{Few-shot} & \\
\midrule
10 & CE & 78 (7) & 87.39 $\pm$ 0.22 & 88.60 $\pm$ 0.37 & 84.57 $\pm$ 0.89 & --- & 5.59 $\pm$ 3.44 \\
 & SRFE & 194 (19) & 87.59 $\pm$ 0.31 & 88.71 $\pm$ 0.41 & 84.97 $\pm$ 0.79 & --- & 4.11 $\pm$ 1.73 \\
\midrule
50 & CE & 102 (10) & 79.88 $\pm$ 0.64 & 88.88 $\pm$ 1.35 & 73.88 $\pm$ 1.22 & --- & 11.48 $\pm$ 3.43 \\
 & SRFE & 204 (20) & 79.88 $\pm$ 0.47 & 88.05 $\pm$ 1.29 & 74.43 $\pm$ 1.19 & --- & 10.18 $\pm$ 2.47 \\
\midrule
100 & CE & 95 (9) & 74.27 $\pm$ 0.49 & 85.78 $\pm$ 2.61 & 74.05 $\pm$ 1.87 & 51.70 $\pm$ 3.76 & 7.82 $\pm$ 4.48 \\
 & SRFE & 260 (26) & 75.08 $\pm$ 0.67 & 86.15 $\pm$ 1.96 & 68.03 $\pm$ 2.10 & 67.05 $\pm$ 5.60 & 8.41 $\pm$ 3.78 \\
\bottomrule
\end{tabular}
\end{table}
\begin{table}[htbp]
\centering
\caption{Hyperparameter optimization evaluation summary on CIFAR-100-LT. Results report the performance of the top-$k$ subset of validation-selected models from completed trials across imbalance factors.}
\label{tab:cifar_100_lt_top10_summary}
\begin{tabular}{llcccccc}
\toprule
\multirow{2}{*}{\textbf{IF}} & \multirow{2}{*}{\textbf{Loss}} & \multirow{2}{*}{\textbf{Trials ($k$)}} & \multicolumn{4}{c}{\textbf{Accuracy (\%)}} & \multirow{2}{*}{\textbf{ECE (\%)}} \\
\cmidrule(lr){4-7}
 & & & \textbf{Overall} & \textbf{Many-shot} & \textbf{Med-shot} & \textbf{Few-shot} & \\
\midrule
10 & CE & 78 (7)   & 58.03 $\pm$ 0.51 & 65.33 $\pm$ 0.52 & 41.00 $\pm$ 1.10 & --- & 2.37 $\pm$ 2.56 \\
 & SRFE & 175 (17) & 57.36 $\pm$ 0.52 & 64.16 $\pm$ 0.76 & 41.50 $\pm$ 1.06 & --- & 11.23 $\pm$ 4.36 \\
\midrule
50 & CE & 81 (8)   & 43.88 $\pm$ 0.44 & 65.71 $\pm$ 1.00 & 36.22 $\pm$ 1.08 & 11.61 $\pm$ 1.81 & 5.28 $\pm$ 3.18 \\
 & SRFE & 169 (16) & 45.52 $\pm$ 0.61 & 65.37 $\pm$ 1.16 & 37.32 $\pm$ 0.86 & 19.00 $\pm$ 1.57 & 31.16 $\pm$ 3.90 \\
\midrule
100 & CE & 82 (8)  & 40.48 $\pm$ 0.77 & 68.40 $\pm$ 0.92 & 39.94 $\pm$ 1.65 & 8.53 $\pm$ 1.08 & 17.59 $\pm$ 7.64 \\
 & SRFE & 214 (21) & 41.13 $\pm$ 0.48 & 69.43 $\pm$ 0.88 & 39.43 $\pm$ 0.86 & 10.10 $\pm$ 1.03 & 16.30 $\pm$ 11.43 \\
\bottomrule
\end{tabular}
\end{table}

\cref{tab:cifar-lt_best-hparams} summarizes the optimal set of hyperparameters for each combination of dataset and IF. The optimal choice of optimizer was SGD and batch size was 32 for all studies so they are omitted from the table.

\begin{table}[htbp]
\centering
\caption{Optimal hyperparameter configurations for SRFE on CIFAR-LT. The optimal optimizer and batch size were SGD and 32, respectively, for all studies.}
\label{tab:cifar-lt_best-hparams}
\begin{tabular}{lccccccc}
\toprule
\textbf{Dataset} & \textbf{IF} & $\bm{\tau}$ & \makecell{\textbf{Label}\\\textbf{Smoothing}} & \makecell{\textbf{Learning}\\\textbf{Rate}} & \makecell{\textbf{Weight}\\\textbf{Decay}} & \textbf{Scheduler} \\
\midrule
CIFAR-10-LT  & 10  & 0.7091 & 0.0295 & 0.1643 & $8.33 \times 10^{-5}$ & MultiStep w/ Warmup \\
             & 50  & 0.8312 & 0.0461 & 0.0375 & $4.52 \times 10^{-4}$ & One Cycle           \\
             & 100 & 0.5839 & 0.0995 & 0.0831 & $3.02 \times 10^{-4}$ & MultiStep w/ Warmup \\
\midrule
CIFAR-100-LT & 10  & 0.5841 & 0.0629 & 0.0245 & $3.00 \times 10^{-4}$ & MultiStep w/ Warmup \\
             & 50  & 0.6548 & 0.0039 & 0.0294 & $4.53 \times 10^{-4}$ & MultiStep w/ Warmup \\
             & 100 & 0.7598 & 0.0718 & 0.0150 & $4.68 \times 10^{-3}$ & One Cycle           \\
\bottomrule
\end{tabular}
\end{table}

\paragraph{Hyperparameter importance}

We calculated an importance score for each hyperparameter using the PED-ANOVA score \citep{watanabe2023ped} \pcref{tab:hparam_importance}. We used the default settings as provided by Optuna \citep{optuna} in our evaluations, which sets the target quantile to the top 10\%.

\begin{table}[htbp]
\centering
\caption{Hyperparameter importance scores (PED-ANOVA) across IFs on CIFAR-10-LT and CIFAR-100-LT. Scores quantify the fraction of validation performance variance explained by each parameter (summing to 1.0 per row). Learning rate is not included as the search space was dynamically set per optimizer.}
\label{tab:hparam_importance}
\begin{tabular}{llcccccc}
\toprule
\textbf{Dataset} & \textbf{IF} & $\bm{\tau}$ & \makecell{\textbf{Label}\\\textbf{Smoothing}} & \makecell{\textbf{Batch}\\\textbf{Size}} & \makecell{\textbf{Weight}\\\textbf{Decay}} & \textbf{Optimizer} & \textbf{Scheduler} \\
\midrule
CIFAR-10  & 10 & 0.2258 & 0.1061 & 0.0865 & \textbf{0.5364} & 0.0343 & 0.0109 \\
          & 50 & \textbf{0.4499} & 0.1394 & 0.1941 & 0.1515 & 0.0413 & 0.0237 \\
          & 100 & \textbf{0.6660} & 0.0225 & 0.1563 & 0.0118 & 0.0149 & 0.1284 \\
\midrule
CIFAR-100 & 10 & \textbf{0.5778} & 0.2454 & 0.0647 & 0.0358 & 0.0406 & 0.0357 \\
          & 50 & \textbf{0.5301} & 0.2339 & 0.0885 & 0.0560 & 0.0561 & 0.0355 \\
          & 100 & 0.1063 & 0.0816 & 0.1072 & 0.0186 & 0.0104 & \textbf{0.6758} \\
\bottomrule
\end{tabular}
\end{table}

On CIFAR-10-LT, $\tau$ increasingly accounts for variations in performance as the severity of imbalance increases (22.58\%, 44.99\%, 66.60\% for IF=10, 50, 100). Interestingly, weight decay is the most dominant determinant of performance for IF=10 at 53.64\%.

On CIFAR-100-LT at IF=10 and 50, $\tau$ alone accounts for the majority of the variation in performance, though label smoothing emerges as an additional factor as well. Combined, $\tau$ and label smoothing account for 82.32\% and 76.40\% of the variation for IF=10 and 50, respectively. At IF=100, however, the scheduler is the most significant determinant and the importance of $\tau$ and label smoothing are significantly lower.

\paragraph{Label smoothing} \label{app:cifar-results-label-smoothing}

We further validated our label smoothing power law as derived in \cref{app:label-smoothing} on our hyperparameter optimization results \pcref{fig:cifar-label-smoothing-power-law}. We compare the distribution of label smoothing values used in the top 10\% of trials in the hyperparameter optimization study with the label smoothing value derived from our power law.

\begin{figure}
    \centering
    \includegraphics[width=0.8\linewidth]{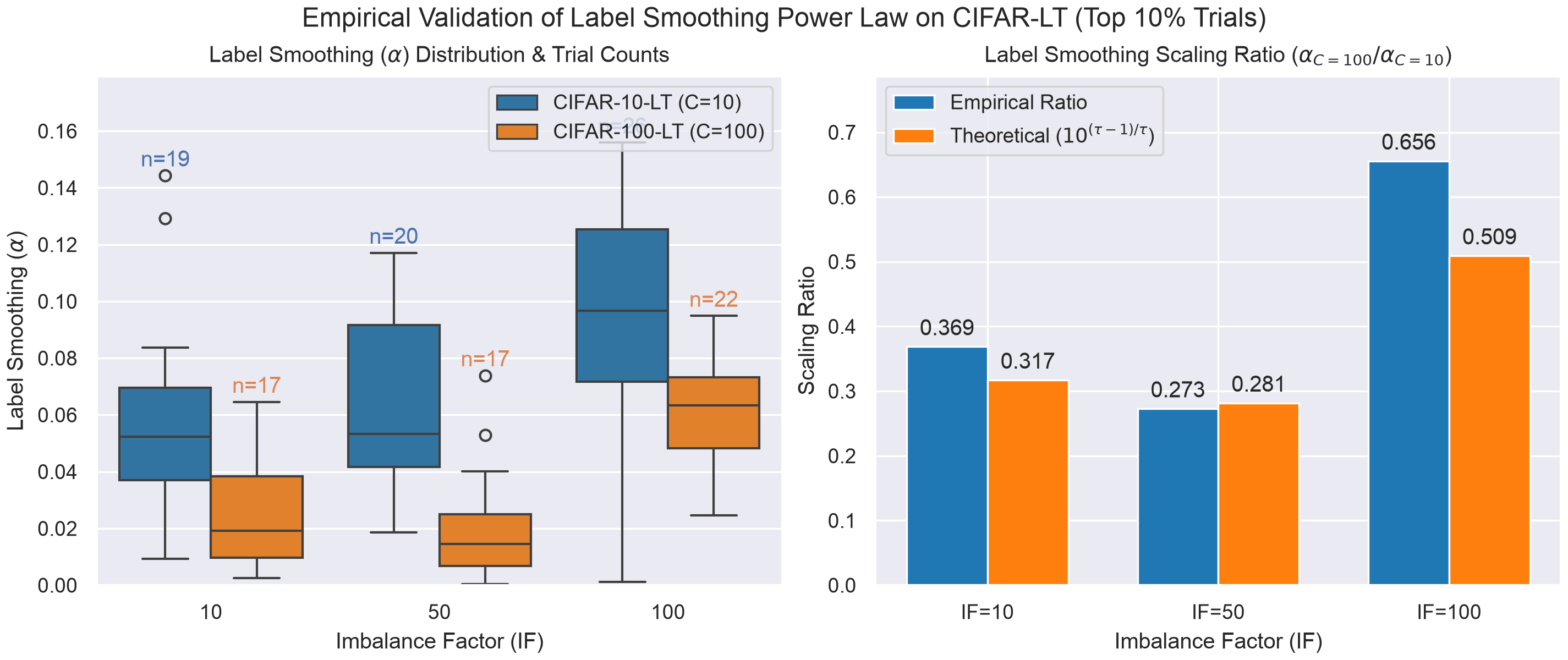}
    \caption{Label smoothing power law validation for CIFAR-LT.}
    \label{fig:cifar-label-smoothing-power-law}
\end{figure}

\section{Implementation} \label{app:implementation}

Below, we describe how the SRFE objective can be estimated and implemented in continuous and discrete settings, and provide PyTorch-style pseudocode for both cases. Please refer to our repository for implementations in PyTorch.

\subsection{Continuous case}

In continuous space, the Chernoff $\tau$-coefficient \cref{eq:F-tau} is estimated using Monte Carlo sampling. By sampling $x\sim Q$, we can express \cref{eq:F-tau} as an expectation over $Q$:
\begin{equation*}
    F(\tau)
    = \int_\gX q(x) \left(\dfrac{p(x)}{q(x)}\right)^\tau d\mu
    = \E_{x\sim Q}\left[\left(\dfrac{p(x)}{q(x)}\right)^\tau\right].
\end{equation*}
For $N$ samples $\{x_i\}_{i=1}^N$, we can estimate the expectation with the estimator
\begin{align*}
    \widehat{F(\tau)}
    &= \dfrac1N \sum_{i=1}^N \left(\dfrac{p(x_i)}{q(x_i)}\right)^\tau \\
    &= \dfrac1N \sum_{i=1}^N \exp\left(\log\left(\dfrac{p(x_i)}{q(x_i)}\right)^\tau\right) \\
    &= \dfrac1N \sum_{i=1}^N \exp\left(\tau(\log p(x_i) - \log q(x_i))\right),
\end{align*}
where we use the log-probabilities instead of the raw probabilities for numerical stability (as is common in deep learning).
Substituting into \cref{eq:srfe}, we get
\begin{align*}
    \SRFE{P}{Q}
    &= -\dfrac{1}{\tau(1-\tau)} \log \left(\dfrac1N \sum_{i=1}^N \exp\left(\tau(\log p(x_i) - \log q(x_i))\right) \right) \\
    &= -\dfrac{1}{\tau(1-\tau)} \left(\log\dfrac1N + \log\sum_{i=1}^N \exp\left(\tau(\log p(x_i) - \log q(x_i))\right) \right),
\end{align*}
where the second term in the parentheses can be calculated efficiently with numerical stability using a function like \texttt{torch.logsumexp}.
We provide pseudocode for a PyTorch-style implementation below.

\begin{algorithm}[htbp]
\caption{Continuous SRFE Loss (Vectorized)}
\label{alg:srfe_continuous_vectorized}
\KwIn{Predicted log-probabilities $\vq \in \R^N$, Target log-probabilities $\vp \in \R^N$, SRFE parameter $\tau \in (0,1)$}
\KwOut{SRFE loss value $\gL_{\text{SRFE}}$}

\BlankLine
$\vr \gets \vp - \vq$ \tcp*{Element-wise log density ratio}

\BlankLine
$\log F(\tau) \gets -\log N + \text{logsumexp}(\tau \cdot \vr)$ \tcp*{Numerical stability handled internally}

\BlankLine
$\gL_{\text{SRFE}} \gets -\frac{\log F(\tau)}{\tau(1-\tau)}$\;

\BlankLine
\Return{$\gL_{\text{SRFE}}$}
\end{algorithm}

\subsection{Discrete case}

Consider the $C$-class classification setting with $N$ samples $\{x_i\}_{i=1}^N$. Each sample will have an associated target distribution $\vp_i\in\R^C$, typically a one-hot encoding, and predicted distribution $\vq_i\in\R^C$. Then the Chernoff $\tau$-coefficient for the $i$-th sample is given by
\begin{equation*}
    F_i(\tau) = \sum_{c=1}^C p_{i,c}^\tau q_{i,c}^{1-\tau}
\end{equation*}
where $p_{i,c}$ and $q_{i,c}$ are the true and predicted probabilities for sample $i$ and class $c$, and $\sum_c p_{i,c} = \sum_c q_{i,c} = 1$.
As in the continuous case, we work with log-probabilities, thus
\begin{align*}
    F_i(\tau)
    &= \sum_{c=1}^C \exp(\log(p_{i,c}^\tau q_{i,c}^{1-\tau})) \\
    &= \sum_{c=1}^C \exp(\tau\log p_{i,c} + (1-\tau)\log q_{i,c}) \\
    &= \sum_{c=1}^C \exp(r_{i,c}),
\end{align*}
where $r_{i,c} = \tau\log p_{i,c} + (1-\tau)\log q_{i,c}$.
Substituting into the SRFE expression, we get the SRFE for the $i$-th sample
\begin{equation*}
    \ell_i = -\dfrac{1}{\tau(1-\tau)} \log \sum_{c=1}^C \exp(r_{i,c}),
\end{equation*}
where the log-sum-exp can be efficiently calculated with \texttt{torch.logsumexp}.
For $N$ samples, we calculate the average of the per-sample losses over the batch
\begin{equation*}
    \gL_{\text{SRFE}}
    = \dfrac1N \sum_{i=1}^N \ell_i
    = -\dfrac{1}{N\tau(1-\tau)} \sum_{i=1}^N \log \sum_{c=1}^C \exp(r_{i,c}).
\end{equation*}
We provide pseudocode for a PyTorch-style implementation below. Note that we add label-smoothing to help with training optimization.

\begin{algorithm}[htbp]
\caption{Discrete SRFE Batch Loss (Vectorized)}
\label{alg:discrete_srfe_batch_vectorized}
\KwIn{Predicted log-probabilities $\mQ \in \R^{N \times C}$, Target labels $\mathbf{y} \in \{1,\dots,C\}^N$, SRFE parameter $\tau \in (0,1)$, Label smoothing parameter $\epsilon \in [0,1)$}
\KwOut{Batch-averaged SRFE loss value $\gL_{\text{SRFE}}$}

\BlankLine
$\mathbf{Y} \gets \text{one\_hot}(\mathbf{y}) \in \R^{N \times C}$ \tcp*{One-hot target labels}
$\mP^{\text{prob}} \gets \mathbf{Y} \cdot (1 - \epsilon) + \frac{\epsilon}{C}$ \tcp*{Label smoothing}
$\mP \gets \log\left(\text{clamp}(\mP^{\text{prob}}, \text{min}=10^{-15})\right)$ \tcp*{Convert to log-probabilities}

\BlankLine
$\mR \gets \tau \cdot \mP + (1 - \tau) \cdot \mQ$ \tcp*{Shape: $N \times C$}

\BlankLine
$\log \mathbf{F}(\tau) \gets \text{logsumexp}(\mR, \text{dim}=1)$ \tcp*{Shape: $N$}

\BlankLine
$\boldsymbol{\ell} \gets -\frac{\log \mathbf{F}(\tau)}{\tau(1-\tau)}$ \tcp*{Sample-level losses, shape: $N$}
$\gL_{\text{SRFE}} \gets \text{mean}(\boldsymbol{\ell})$ \tcp*{Average over batch}

\BlankLine
\Return{$\gL_{\text{SRFE}}$}
\end{algorithm}

\end{document}